\documentclass{opt2022} 

\usepackage{dsfont}
\usepackage{algorithmic}
\usepackage{algorithm}
\usepackage{booktabs} 
\usepackage{multirow}

\title[Optimal Complexity in Decentralized Learning over Time-Varying networks]{Optimal Complexity in Non-Convex
Decentralized Learning over \\  Time-Varying Networks}

\newtheorem{assumption}{Assumption}

\newcommand{\EE}{\mathbb{E}}
\newcommand{\NN}{\mathbb{N}}
\newcommand{\PP}{\mathbb{P}}
\newcommand{\RR}{\mathbb{R}}

\newcommand{\cA}{{\mathcal{A}}}
\newcommand{\cC}{{\mathcal{C}}}
\newcommand{\cF}{{\mathcal{F}}}
\newcommand{\cI}{{\mathcal{I}}}
\newcommand{\cJ}{{\mathcal{J}}}
\newcommand{\cN}{{\mathcal{N}}}
\newcommand{\cO}{{\mathcal{O}}}
\newcommand{\cS}{{\mathcal{S}}}

\newcommand{\cV}{{\mathcal{V}}}
\newcommand{\cW}{{\mathcal{W}}}
\newcommand{\cX}{{\mathcal{X}}}

\newcommand{\va}{{\mathbf{a}}}
\newcommand{\vb}{{\mathbf{b}}}

\newcommand{\vg}{{\mathbf{g}}}
\newcommand{\vh}{{\mathbf{h}}}

\newcommand{\vx}{{\mathbf{x}}}

\newcommand{\vz}{{\mathbf{z}}}

\newcommand{\vW}{\mathbf{W}}

\newcommand{\bxi}{\boldsymbol{\xi}}

\newcommand{\eg}{\emph{e.g.}\xspace}
\newcommand{\ie}{\emph{i.e.}\xspace}

\newcommand{\one}{\mathds{1}}
\newcommand{\prog}{\mathrm{prog}}
\newcommand{\dist}{\mathrm{dist}}
\newcommand{\diam}{\mathrm{diam}}

\optauthor{%
\Name{Xinmeng Huang} \Email{xinmengh@sas.upenn.edu}\\
\addr 
Graduate Group in Applied Mathematics and Computational Science\\
University of Pennsylvania
\AND
\Name{Kun Yuan} \Email{kunyuan@pku.edu.cn}\\
\addr Center for Machine Learning Research \\
Peking University%
}

\begin{document}

\maketitle

\begin{abstract}%
Decentralized optimization with time-varying networks is an emerging paradigm in machine learning. It saves remarkable communication overhead in large-scale deep training
and is more robust in wireless scenarios especially when nodes are moving. Federated learning can also be regarded as decentralized optimization with time-varying communication patterns alternating between global  averaging and local updates.

While numerous studies exist to clarify its  theoretical limits and develop efficient algorithms, it remains unclear what the optimal complexity is for non-convex decentralized stochastic optimization over time-varying networks. The main difficulties lie in how to gauge the effectiveness when transmitting messages between two nodes via time-varying communications, and how to establish the lower bound when the network size is fixed (which is a prerequisite in stochastic optimization). This paper resolves these challenges and establish the first lower bound complexity. We also develop a new decentralized algorithm to nearly attain the lower bound, showing the tightness of the lower bound and the optimality of our algorithm.


\end{abstract}


\section{Introduction}
\label{sec:introduction}

\textbf{Decentralized optimization.} 
Decentralized optimization is an emerging learning paradigm in which each node only communicates with its immediate neighbors per iteration. By avoiding the central server and maintaining a more balanced communication between each pair of connected nodes, decentralized approaches can significantly speedup the training process of large-scale machine learning models \cite{assran2019stochastic,gan2021bagua,bluefog}. Although decentralized optimization has been extensively studied in literature, 
its performance limits with {\em time-varying}  communication patterns has not been fully explored. This paper provides a better understanding in optimal complexity for non-convex decentralized stochastic optimization over time-varying communication networks. 

\vspace{1mm}
\noindent \textbf{Time-varying communication pattern.} Decentralized optimization over time-varying communication networks is ubiquitous in applications. In large-scale deep neural network training, sparse and time-varying network topologies such as one-peer exponential graph \cite{assran2019stochastic,ying2021exponential} and EquiRand \cite{song2022simple} endow decentralized learning with a state-of-the-art balance between communication efficiency and convergence rate. 
In wireless signal processing, time-varying topologies naturally emerge when the nodes (such as cellphones, drones, robots, etc.) are moving \cite{tu2010foraging,tu2011mobile}. Federated learning \cite{mcmahan2017communication,stich2019local} can  also be regarded as a special decentralized learning paradigm which admits a time-varying communication pattern alternating between global averaging and local updates.

\vspace{1mm}
\noindent \textbf{Prior results in theoretical limits.} A series of pioneering works have attempted to establish the optimal complexity in decentralized optimization over {\em static} communication networks. In the deterministic regime, \cite{scaman2017optimal,scaman2018optimal,sun2019distributed} clarified the theoretical limits and proposed algorithms to (nearly) attain these limits. In the stochastic regime, recent works \cite{lu2021optimal,yuan2022revisiting} have established the optimal complexity in the non-convex setting. However, there are few studies on theoretical limits in decentralized optimization over {\em time-varying} communication networks. A recent useful work \cite{kovalev2021lower} establishes the optimal complexity over time-varying networks for deterministic and strongly-convex problems. While this bound is inspiring, its analysis, as well as all existing results in literature to our knowledge, cannot be easily extended to the stochastic setting due to challenges  below.  

\vspace{1mm}
\noindent \textbf{Challenges.} 
When considering a static network topology, it is known that the optimal complexity in decentralized optimization is typically proportional to diameter $D$ of the static topology \cite{scaman2017optimal}. 
Clarifying how the diameter $D$ affects the algorithmic convergence is the key to justifying the influence of the communication network on the optimal convergence rate.  
However, it is unclear in literature how to gauge, or even define, the graph diameter for a sequence of time-varying networks.

Furthermore, this paper considers decentralized stochastic optimization where {\em the network size $n$ is a fixed constant}. A fixed $n$ is a prerequisite in distributed stochastic optimization which enables distributed algorithms to achieve the linear speedup in convergence rate $O(\sigma/\sqrt{nT})$ where $\sigma$ indicates the gradient noise and $T$ is the algorithmic iteration.
In decentralized deterministic optimization, however, size $n$ does not appear in the convergence rate. Thus, it does not need to be fixed  
and can be varied freely to simplify the lower-bound analysis. In fact, references  \cite{scaman2017optimal,scaman2018optimal,kovalev2021lower} all tune $n$ delicately to derive the optimal complexity for decentralized deterministic optimization over static or time-varying  networks. Therefore, the analysis in \cite{scaman2017optimal,scaman2018optimal,kovalev2021lower} cannot be extended to decentralized stochastic setting in which the network size $n$ is fixed. 


\vspace{1mm}
\noindent \textbf{Main results.} This paper overcomes the above two challenges and successfully establishes the optimal complexity for decentralized stochastic optimization over time-varying network topologies. 

\begin{itemize}
\item Inspired by the graph diameter of a static network topology, we introduce a novel {\em effective graph diameter} to gauge how efficient a message is transmitted between two farthest nodes via a sequence of time-varying decentralized communications. 

\item We provide the first lower bound complexity for decentralized non-convex  stochastic optimization over time-varying networks. The derivation of this lower bound is based on a novel family of {\em sun-shaped network topologies}. Given any fixed network size $n$, we can always construct a sequence of time-varying {sun-shaped topologies} that maintains the optimal relation between the effective graph diameter and the network connectivity. 

\item We prove that the established lower bound complexity can be nearly attained (up to logarithmic factors) by integrating multiple gossip communications \cite{liu2011accelerated,rogozin2021towards,yuan2021removing} and gradient accumulation \cite{scaman2017optimal,Rogozin2021AnAM,lu2021optimal} to the vanilla stochastic gradient tracking approach \cite{nedic2017achieving,di2016next,qu2018harnessing,xu2015augmented,lu2019gnsd}. It implies that our complexity bound is tight and the proposed  algorithm is nearly optimal.
\end{itemize} 
All established results in this paper as well as those of existing state-of-the-art decentralized learning algorithms over time-varying networks are listed in Table \ref{tab:nc-compare}.

\begin{table}[t]
\centering 
\caption{\small Rate comparison between different decentralized stochastic algorithms over time-varying networks. Parameter $n$ denotes the number of all computing nodes, $\beta\in [0,1)$ denotes the connectivity measure of the weight matrix, $\sigma^2$ measures the gradient noise, $b^2$ denotes data heterogeneity, and $T$ is the number of iterations. Other constants such as the initialization $f(x^{(0)}) - f^\star$ and smoothness constant $L$ are omitted for clarity. Logarithm factors are hidden in the $\tilde{O}(\cdot)$ notation.
}
\begin{tabular}{clll}
\toprule
      \textbf{Bound type}              & \textbf{Reference}                                                                           & \textbf{Gossip matrix}                     & \hspace{10mm}\textbf{Convergence rate}    \\ \midrule
      \vspace{1mm}
Lower & \textcolor{blue}{Theorem} \ref{thm:lower-bound-nc}                                               & \textcolor{blue}{$\beta \in[0,1-\frac{1}{n}]$} &         \textcolor{blue}{$\Omega\big( \frac{\sigma}{\sqrt{nT}}+ \frac{1}{T (1-\beta)} \big)$}                 \\\midrule
\multirow{5}{*}{Upper} & DSGD \cite{koloskova2020unified}                                    & $\beta \in [0, 1)$           &     $O\big( \frac{\sigma}{\sqrt{nT}} \hspace{-1mm}+\hspace{-1mm} \frac{\sigma^{\frac{2}{3}}}{T^{\frac{2}{3}}(1-\beta)^{\frac{1}{3}}}+\frac{b^{\frac{2}{3}}}{T^{\frac{2}{3}}(1-\beta)^{\frac{2}{3}}} \hspace{-1mm}+ \frac{1}{T(1-\beta)} \big)$           \\
                             & DSGT \cite{xin2020improved} & $\beta \in [0, 1)$           & $\tilde{O}\big( \frac{\sigma}{\sqrt{nT}} + \frac{\sigma^{\frac{2}{3}}}{T^{\frac{2}{3}}(1-\beta)}+\frac{1}{T(1-\beta)^2}  \big)$     \\
                             & {\color{blue}MC-DSGT}                                                                              & {\color{blue}$\beta \in [0, 1)$}           & {\color{blue}$\tilde{O}\big( \frac{\sigma}{\sqrt{nT}}+ \frac{1}{T (1-\beta)}\big)$}                      \\ \bottomrule

\end{tabular}
\vspace{-4mm}
\label{tab:nc-compare}
\end{table}

\vspace{1mm}
\noindent \textbf{Other related works.} Decentralized optimization can be tracked back to \cite{tsitsiklis1986distributed}.  Decentralized gradient descent \cite{nedic2009distributed,yuan2016convergence,lian2017can}, diffusion \cite{chen2012diffusion,sayed2014adaptive} and dual averaging \cite{duchi2011dual} are early popular decentralized methods. Other advanced variants  extend decentralized methods to data-heterogeneous scenarios \cite{tang2018d,xin2020improved,lu2019gnsd,alghunaim2021unified,koloskova2021improved}, adaptive momentum settings \cite{lin2021quasi,yuan2021decentlam,nazari2019dadam}, or asynchronous implementations \cite{lian2018asynchronous}. When the network topology is time-varying, reference \cite{kovalev2021lower} establishes optimal convergence rate under the deterministic and strongly-convex setting. References \cite{kovalev2021lower,li2021accelerated} develop decentralized methods with Nesterov acceleration to nearly achieve such optimal convergence rate. In the stochastic and non-convex setting, the convergence rate of decentralized SGD over general time-varying networks is clarified in \cite{koloskova2020unified}. Other references  \cite{ying2021exponential,song2022simple,wang2019matcha} study specific sparse and time-varying network topologies that can further save communication overheads in decentralized SGD. However, none of these works provides the optimal complexity for non-convex decentralized learning over time-varying networks.




\section{Problem setup}
\label{sec-setup}
\textbf{Problem setup.} Consider the following problem with a network of $n$ computing nodes:
\begin{align}\label{dist-opt}
{\color{black} \min_{x \in \mathbb{R}^d}\  f(x)=\frac{1}{n}\sum_{i=1}^n f_i(x) \quad \mbox{where} \quad f_i(x): = \mathbb{E}_{\xi_i \sim D_i} [F(x;\xi_i)].}
\end{align}
Function $f_i(x)$ is local to node $i$, and random variable $\xi_i$ denotes the local data that follows distribution $D_i$. Each local data distribution $D_i$ can be different across all nodes. 

\paragraph{Assumptions.} The optimal convergence rate is established under the following assumptions. 

\begin{itemize}
    \item \textbf{Function class.} We let the {function class $\cF_{L}^\Delta$} denote the set of all functions satisfying the following assumption for any dimension $d\in \NN_+$ and initialization point $x^{(0)}\in\RR^d$.
    \begin{assumption}[\sc Cost functions] \label{asp:nc}
	We assume each $f_{i}$ has $L$-Lipschitz gradient, {\ie},
	$$
	\left\|\nabla f_{i}(x)-\nabla f_{i}(y)\right\| \leq L\|x-y\|
	$$
	for all $i\in[n]$, $x, y \in \mathbb{R}^{d}$, and $f(x^{(0)})-\inf _{x \in \mathbb{R}^{d}} f(x)\leq \Delta $ with $f=\frac{1}{n}\sum_{i=1}^n f_i$.
    \end{assumption}

    \item \textbf{Gradient oracle class. } 
We assume each worker $i$ has  access to its local gradient $\nabla f_i(x)$ via a stochastic gradient oracle $O_i(x;\zeta_i)$ subject to independent randomness $\zeta_i$, \eg, the mini-batch sampling $\zeta_i\triangleq\xi_i\sim D_i$. We further assume that the output $O_i(x,\zeta_i)$ is an  unbiased estimator of the full-batch gradient $\nabla f_i(x)$ with a bounded variance. Formally, we let the {stochastic gradient oracle class $\cO_{\sigma^2}$} denote the set of all oracles $O_i$ satisfying Assumption \ref{asp:gd-noise}.
\begin{assumption}[\sc Gradient stochasticity]\label{asp:gd-noise}
We assume local gradient oracle $O_i$ satisfies
\begin{align*}
\EE_{\zeta_i}[O_i(x;\zeta_i)]=\nabla f_i(x)\quad \text{ and }\quad \EE_{\zeta_i}[\|O_i(x;\zeta_i)-\nabla f_i(x)\|^2]\leq \sigma^2
\end{align*}
 for any $x \in \RR^d$ and $i\in[n]$. 
\end{assumption}
    
\item \textbf{Decentralized communication.} Let $\cV=[n]$ denote the set of $n$ computing nodes. For any communication round $t\geq 0$, we assume nodes are connected through a time-varying communication network represented by a 
graph $G^t=(V,E^t)$, where $E^t\subseteq \{(j,i)\in\cV\times \cV:\,i\neq j\}$ is the set of links activated at round $t$. If a directed link $(j,i) \in E^t$, then node $j$ can transmit information to node $i$ at round $t$. 
In decentralized communication protocols, each node $i$ can only receive messages with its immediate neighbors via links in $E^t$. 

\item \textbf{Weight matrix class.} To characterize the decentralized communication in algorithm development, we associate each time-varying communication graph $G^t$ with a weight matrix $W^t$ (also known as the gossip matrix \cite{nedic2009distributed,yuan2016convergence}). As in \cite{kovalev2021lower,lu2021optimal,yuan2022revisiting}, we consider a sequence of time-varying weight matrices $\{W^t\}_{t=0}^\infty\subseteq\RR^{n\times n}$ satisfying Assumption \ref{asp:weight-matrix}.
\begin{assumption}[\sc Weight matrix]\label{asp:weight-matrix}
    For any $t\geq 0$,  $W^t=[w_{i,j}^t]_{i,j=1}^n$ satisfies
    \begin{enumerate}
    \item if $(j, i) \notin E^t$ and $i\neq j$, then $w^t_{i,j}=0$;
    \item $W^t\mathds{1}_n = \mathds{1}_n$ and $\one_n^\top W^t= \mathds{1}_n^\top$  where $\one_n=[1,\dots,1]^\top\in\RR^n$;
    \item there exists a fixed constant $\beta\in[0,1)$ such that $\|W^t - \mathds{1}_n\mathds{1}_n^\top/n\|_2 \leq  \beta$.
\end{enumerate}
\end{assumption}
Note that a weight matrix $W^t$ satisfying Assumption \ref{asp:weight-matrix} is not necessarily symmetric or positive semi-definite.  The constant $\beta$ is the {\em connectivity measure} that gauges how well the network topology $G^t$ is connected. Constant $\beta \to 0$ (which implies $W^t \to \frac{1}{n}\mathds{1}_n \mathds{1}_n^\top$) indicates a well-connected topology while $\beta \to 1$ (which implies $W^t \to I$) indicates a poor connection. We let $\cW_{n,\beta}$ denote the class of all weight matrices $W^t\in \RR^{n\times n}$ satisfying Assumption \ref{asp:weight-matrix}. 

    \item \textbf{Algorithm class.} We consider an algorithm $A$ in which each node $i$ assesses an unknown local function $f_i $ via the \emph{independent} stochastic gradient oracle $O_i(x;\,\zeta_i) \in \cO_{\sigma^2}$.  Each node $i$ running algorithm $A$ will maintain a local model copy $x^{(t)}_{i}$ at round $t$. 
    We assume $A$ to follow the partial averaging policy, \ie, each node communicates at round $t$ via protocol
    $$
    \setlength{\abovedisplayskip}{3pt}\setlength{\belowdisplayskip}{3pt}z_i = \sum_{j\in \mathcal{N}_i^t} w_{i,j}^t y_j, \quad \forall\,i \in [n]$$ with some $W^t = [w_{i,j}^t]_{i,j=1}^n \in \cW_{n, \beta}$ where  $y$ and $z$ are the input and output of the communication protocol. In addition, we assume $A$ to follow the zero-respecting policy \cite{carmon2020lower,carmon2021lower}. Informally speaking, the zero-respecting policy requires that the number of non-zero entries of local model copy $x_i^{(t)}$ can only be increased by either sampling its own stochastic gradient oracle or interacting with the neighboring nodes.
    We let $\cA_{\{W^t\}_{t=0}^\infty}$ be the set of all algorithms following the partial averaging and zero-respecting policies. 
\end{itemize}

With the above classes, this paper will clarify the following question: {\em Given loss functions $\{f_i\}_{i=1}^n\subseteq \cF_{L}^\Delta$, stochastic gradient oracles $\{O_i\}_{i=1}^n \subseteq  \cO_{\sigma^2}$, a sequences of time-varying networks $\{G^t\}_{t=0}^{\infty}$ and its associated weight matrices $\{W^t\}_{t=0}^{\infty}\subseteq \cW_{n,\beta}$, what is the optimal complexity to solve problem \eqref{dist-opt}, and what decentralized algorithm $A\in \cA_{\{W^t\}_{t=0}^\infty}$ can achieve it?}

\vspace{1mm}
\noindent \textbf{Notations.} We let $[n] := \{1,2,\cdots, n\}$. 
For any network $G=([n],E)$ and node $i\in[n]$, we let $\cN_G(i)$  denote $\{j:(j,i)\in E\text{ or }j =i\}$, \ie, the neighborhood set of node $i$ in network $G$. 
Similarly, for a subset of nodes $\cI\subseteq[n]$, we use $\cN_G(\cI)$ to denote its neighborhood set $\cup_{i\in\cI}\cN_G(i)$.



\section{Sun-shaped graphs and effective distance/diameter}
As we have discussed in the {\bf Challenge} paragraph in Section \ref{sec:introduction}, it is unknown in literature (1) how to gauge the graph diameter for a sequence of time-varying network topologies, and (2) how to develop time-varying network topologies that can maintain the optimal relation between graph diameter and the network connectivity when the network size $n$ is fixed. This section will resolve these two challenges by introducing a novel family of sun-shaped time-varying graphs.

\begin{definition}[\sc Sun-shaped  graph]\label{def:rl}
    Given any positive integers $n\geq 2$ and $\cC\subseteq[n]$, the sun-shaped graph over nodes $[n]$ with center set $\cC$, denoted by $\cS_{n,\cC}$, is an undirected graph in which the neighborhood $\cN_{\cS_{n,\cC}}(i)$ of node $i \in [n]$ is given by 
    \begin{equation*}
        \cN_{\cS_{n,\cC}}(i)=\begin{cases}
        [n]& \text{if }i\in\cC;\\
        \cC\cup\{i\}&\text{otherwise}.
        \end{cases}
    \end{equation*}
\end{definition}
The center set $\cC$ in $\cS_{n,\cC}$ constitutes a complete subgraph. Nodes in the complete set $[n]\backslash\cC$ are connected to each node in $\cC$, but there is no connection between any pair of nodes in  $[n]\backslash\cC$. Note that a sun-shaped graph $\cS_{n,\cC}$  with $|\cC|=1$ corresponds to a star graph while $|\cC|=n$ or $|\cC|=n-1$ corresponds to a  complete graph. $\cS_{n,\cC}$ can be regarded as an intermediate state between the star and complete graphs when $2 \le |\cC| \le n-2$, see the illustration in Figure \ref{fig:sun-graph}.

\begin{figure}[!t]
    \centering
    \includegraphics[width=0.24\textwidth]{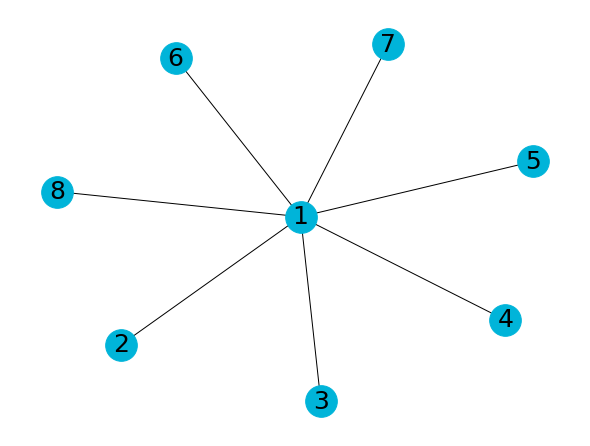}
    \includegraphics[width=0.24\textwidth]{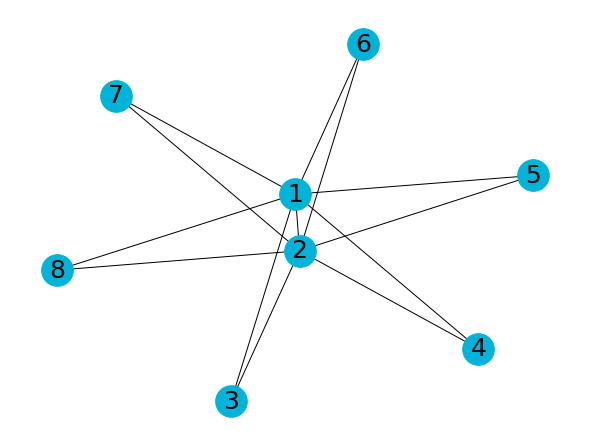}
    \includegraphics[width=0.24\textwidth]{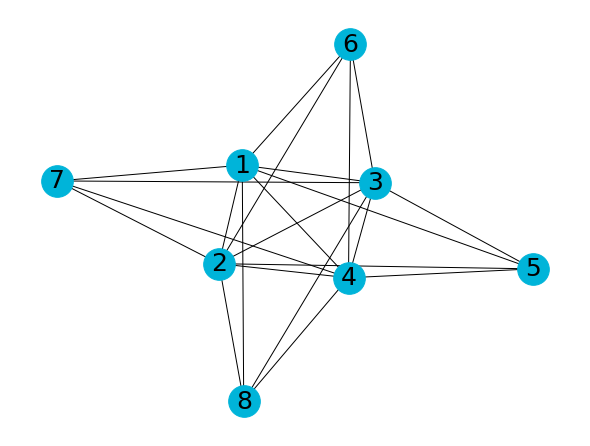}
    \includegraphics[width=0.24\textwidth]{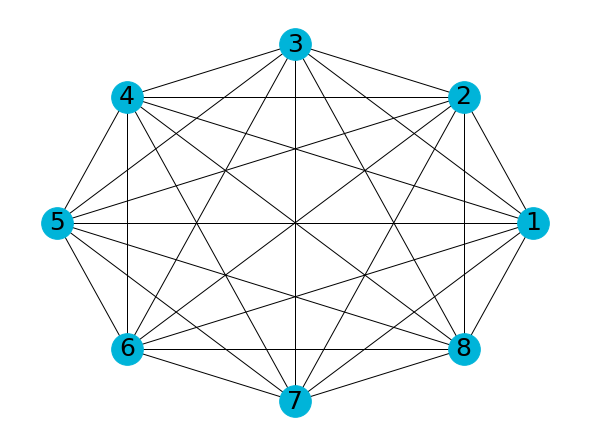}
    \caption{\small An illustration of the sun-shaped graph with size $8$ and center sets $[1],[2],[4],[7]$ (or $[8]$). It is observed that $\cS_{8,[1]}$ is a star graph while $\cS_{8,[8]}$ is a complete graph.}
    \label{fig:sun-graph}
    \vskip -7mm
\end{figure}

We next introduce effective graph diameter to gauge how efficient a message is transmitted between two farthest nodes via a sequence of time-varying decentralized communications. 
\begin{definition}[\sc Effective distance/diameter]
    We define the \emph{effective distance} $\dist_{\{G^t\}_{t=0}^\infty}(i,j)$ between two nodes $i\neq j$ over a sequence of networks $\{G^t\}_{t=0}^\infty$ to be the smallest number of rounds with which a message sent from node $i$ or $j$ at some round $t$ can be received by the other one via decentralized communications (\ie, communicating over $\{G^{t^\prime}\}_{t^\prime=t}^\infty$). Formally, we define
    \begin{align*}
        \dist_{\{G^t\}_{t=0}^\infty}(i,j):=\max\Big\{&\arg\min_{R}\{R:j\in {\cN}_{G^{t}}({\cN}_{G^{t+1}}(\cdots{\cN}_{G^{t+R-1}}(i) \cdots))\text{ for some }t\geq 0\},\\
        &\arg\min_{R}\{R:i\in {\cN}_{G^{t}}({\cN}_{G^{t+1}}(\cdots{\cN}_{G^{t+R-1}}(j) \cdots))\text{ for some }t\geq 0\}\Big\}.
    \end{align*}
    Similarly, we define the effective distance between two disjoint subsets of nodes $\cI_1,\,\cI_2\subsetneq[n]$ as
    \begin{equation*}
        \dist_{\{G^t\}_{t=0}^\infty}(\cI_1,\cI_2)=\min_{i\in\cI_1,\,j\in\cI_2}\{\dist_{\{G^t\}_{t=0}^\infty}(i,j)\}.
    \end{equation*}
    We define the {\em effective diameter} to be the largest effective distance between any two nodes, \ie, 
    \begin{equation*}
        \diam_{\{G^t\}_{t=0}^\infty}:=\max_{1\leq i\neq j\leq n}\{\dist_{\{G^t\}_{t=0}^\infty}(i,j)\}.
    \end{equation*}
\end{definition}
The definitions of effective distance and effective diameter are specific to the time-varying networks. We remark that when the networks are static, \ie, $G^t=G$ for any $t\geq 0$, then the effective distance/diameter reduces to  the canonical distance/diameter in a  static graph.


The following fundamental theorem establishes the relation between the effective distance with respect to a sequence of sun-shaped graphs and the connectivity measure $\beta$. 

\begin{theorem}\label{thm:sun-key}
Given a fixed $n\geq 2$, two disjoint subsets of nodes $\cI_1,\,\cI_2\subsetneq[n]$, and any $\beta \in [0, 1-\frac{1}{n}]$, there exists a sequence of sun-shaped graphs $\{\cS_{n,\cC^t}\}_{t=0}^\infty$ such that
\begin{itemize}
\item[(1)] the graph $\cS_{n,\cC^t}$ at round $t$ has an associated weight matrix $W^t \in \cW_{n,\beta}$, \ie, $W^t \in \RR^{n\times n}$, $\one_n^\top W^t=\one_n^\top$, $W^t\one_n=\one_n$, and $\|W^t - \frac{1}{n}\mathds{1}_n\mathds{1}_n^\top\|_2 \leq  \beta$;
\item[(2)] the effective distance between $\cI_1$ and $\cI_2$ satisfies $$\dist_{\{\cS_{n,\cC^t}\}_{t=0}^\infty}(\cI_1,\cI_2)= \Theta\left(\frac{1-(|\cI_1|+|\cI_2|)/n}{1-\beta}+1\right);$$ In particular, if $1-(|\cI_1|+|\cI_2|)/n=\Omega(1)$, then $\dist_{\{\cS_{n,\cC^t}\}_{t=0}^\infty}(\cI_1,\cI_2)=\Theta((1-\beta)^{-1})$.
\end{itemize}
\end{theorem}

\section{Lower Bound}
With the help of Theorem \ref{thm:sun-key}, we are ready to establish the lower bound for non-convex decentralized stochastic optimization over time-varying networks. All proof details are in Appendix \ref{app:lower-bounds}.
\begin{theorem}\label{thm:lower-bound-nc}
For any $L>0$, $n\geq 2$, $\beta \in[0,1-\frac{1}{n}]$,  and $\sigma>0$, there exists a set of loss functions $\{f_i\}_{i=1}^n \subseteq \mathcal{F}_{L}^\Delta$, a set of stochastic gradient oracles $\{O_i\}_{i=1}^n \subseteq \mathcal{O}_{\sigma^{2}}$, and a sequence of weight matrices $\{W^t\}_{t=0}^\infty \subseteq \mathcal{W}_{n,\beta}$ resulted from the sun-shaped graphs, such that it holds for the output $\hat{x}$ of any  $A \in \mathcal{A}_{\{W^t\}_{t=0}^\infty}$ starting form $x^{(0)}$ that 
\begin{align}\label{eqn:lower-bound-nc}
  \mathbb{E}[\|\nabla f(\hat{x})\|^2] = \Omega\left(\left(\frac{\Delta L\sigma^2}{nT}\right)^\frac{1}{2}+ \frac{\Delta L}{T(1-\beta)}\right).
\end{align}
\end{theorem}
\begin{remark}
While the lower bound is established for $\beta \in [0,1-1/n] \subset [0, 1)$, it approaches to $[0, 1)$ as $n$ goes large. 
 Such interval is broad enough to cover most weight matrices (generated through the Laplacian rule $W = I - L/d_{\max}$) resulted from common topologies such as grid, torus, hypercube, exponential graph, complete graph, Erdos-Renyi graph, geometric random graph, etc. whose $\beta$ lies in the interval $[0, 1-1/n]$ when $n$ is sufficiently large. 
 \end{remark}

\section{Upper Bound}


This section presents a decentralized algorithm that achieves the lower bound established in Theorem \ref{thm:lower-bound-nc} up to logarithmic factors. The new algorithm is a direct extension of the vanilla decentralized stochastic gradient tracking (DSGT) \cite{xin2020improved,lu2019gnsd}. Inspired by the algorithm development in \cite{lu2021optimal,kovalev2021lower}, we add two additional components to DSGT: gradient accumulation and multiple-consensus communication. The main recursions are listed in Algorithm \ref{alg:MG-DSGT} which utilizes the fast gossip average step \cite{liu2011accelerated} in Algorithm \ref{alg:MG}. We call the new algorithm as MC-DSGT where ``MC'' indicates ``multiple consensus''. All proofs are in Appendix \ref{app:uppe-bound}. 
\begin{algorithm}[t]
	\caption{Decentralized Stochastic Gradient Tracking with Multiple Consensus (MC-DSGT)}
	\label{alg:MG-DSGT}
	\begin{algorithmic}
		\STATE \noindent {\bfseries Input:} Initialize $x_i^{(0)}=x^{(0)}$ and $h_i^{(0)}=\frac{1}{nR}\sum_{i=1}^n\sum_{r=0}^{R-1}O_i(x^{(k+1)};\zeta_i^{(k+1,r)})$ for any $i\in [n]$; initialize $\vx^{(0)} = [x_1^{(0)},\cdots,x_n^{(0)}]^\top$, $\vh^{(0)} = [h_1^{(0)},\cdots,h_n^{(0)}]^\top$,  and $\tilde{\vg}^{(0)} = \vh^{(0)}$;
  the decentralized gossip communication rounds $R$
		\FOR{$k=0,\cdots,K-1$}
		\STATE Update $\vx^{(k+1)}=\textbf{Multi-Consensus}(\vx^{(k)}-\gamma \vh^{(k)}, 2kR, (2k+1)R)$
		\STATE Query stochastic gradients $\tilde{g}_i^{(k+1)}=\frac{1}{R}\sum_{r=0}^{R-1}O_i(x_i^{(k+1)};\zeta_i^{(k+1,r)})$ at each node $i$
		\STATE Update $\vh^{(k+1)}=\textbf{Multi-Consensus}(\vh^{(k)}+ \tilde{\vg}^{(k+1)}-\tilde{\vg}^{(k)},(2k+1)R, (2k+2)R)$
		\ENDFOR 
	\end{algorithmic}
\end{algorithm}
\begin{algorithm}[t]
	\caption{$\vz^{(t_2)}$ = Multi-Consensus($\vz^{(t_1)}, t_1, t_2$)}
	\label{alg:MG}
	\begin{algorithmic}
		\STATE \noindent {\bfseries Input:} Variable $\vz^{(t_1)}$; index $t_1$ and $t_2$
		\FOR{$t=t_1,\cdots,t_2-1$}
		\STATE Update $\vz^{(t+1)}= W^{t}\vz^{(t)}$
		\ENDFOR 
		\RETURN Variable $\vz^{(t_2)}$
	\end{algorithmic}
\end{algorithm}
Since each node takes $R$ gradient queries and $R$ gossip communications at round $k$, it holds that $T = K R$ when MC-DSGT finishes after $K$ rounds. The following theorems clarify the convergence rate of MC-DSGT where $T = KR$. 

\begin{theorem}\label{thm:MG-DSGT-rate-nc}
Given $L>0$, $n\geq 1$, $\beta \in[0,1)$, $\sigma>0$,  by choosing the learning rate $\gamma$ as in \eqref{eqn:lr}, 
the convergence of Algorithm \ref{alg:MG-DSGT} can be bounded for any $\{f_i\}_{i=1}^n \subseteq \cF_{L}^\Delta$ and any $\{W\}_{t=0}^\infty\subseteq \cW_{n,\beta}$ that
\begin{equation*}
    \frac{1}{K+1}\sum_{k=0}^{K} \mathbb{E}[\|\nabla f(\bar{x}^{(k)})\|^2]=O\left(\left(\frac{\Delta L\sigma^2}{nT}\right)^\frac{1}{2}+\frac{R\Delta L}{T}+\left(\frac{\rho^2\Delta^2 L^2R\sigma^2}{(1-\rho)^3T^2}\right)^\frac{1}{3}+\frac{\rho^2R\Delta L}{T(1-\rho)^2}\right)
\end{equation*}
where $\rho\triangleq \beta^R\in[0,1)$, $\bar{x}^{(k)} = \frac{1}{n}\sum_{i=1}^n x_i^{(k)}$, and $T=KR$ is the total number of gradient queries and gossip communications at each node. If we further set $R$ as in \eqref{eqn:R}, then the rate becomes
\begin{equation}\label{eqn:MG-DSGD-rate-nc}
\frac{1}{K+1}\sum_{k=0}^{K} \mathbb{E}[\|\nabla f(\bar{x}^{(k)})\|^2] = \tilde{O}\left(\left(\frac{\Delta L\sigma^2}{nT}\right)^\frac{1}{2}+ \frac{\Delta L}{T(1-\beta)}\right).
\end{equation}
\end{theorem}
The rate \eqref{eqn:MG-DSGD-rate-nc} matches with the lower bound \eqref{eqn:lower-bound-nc} up to logarithm factors. Therefore, our established lower bound is tight and hence optimal. The comparison between MC-DSGT with other state-of-the-art algorithms for non-convex decentralized stochastic optimization is listed in Table \ref{tab:nc-compare}.

\section{Experiments}
We consider the logistic regression with a non-convex regularization term \cite{xin2020improved,Antoniadis2011PenalizedLR}. The problem formulation is given by $\min_x\frac{1}{n}\sum_{i=1}^nf_i(x)+\rho r(x)$ where
\begin{equation}\label{eqn:nvwegdv}
    f_i(x)=\frac{1}{m}\sum_{j=1}^m \ln(1+\exp(-y_{i,j}\langle h_{i,j},x\rangle)),\quad r(x)=\sum_{k=1}^d\frac{[x]_k^2}{1+[x]_k^2},
\end{equation}
$[x]_k$ denotes the $k$-the entry of $x\in\RR^d$, $\{(h_{i,j},y_{i,j})\}_{j=1}^m$ is the local dataset at node $i$ where $h_{i,j}\in\RR^d$, $y_{i,j}\in\{\pm1\}$ 
is a feature vector and  label, respectively.  The regularization $r(x)$ is a smooth but non-convex function and $\rho>0$ is the regularization weight.

\begin{figure}[!t]
    \centering
    \includegraphics[width=0.3\textwidth]{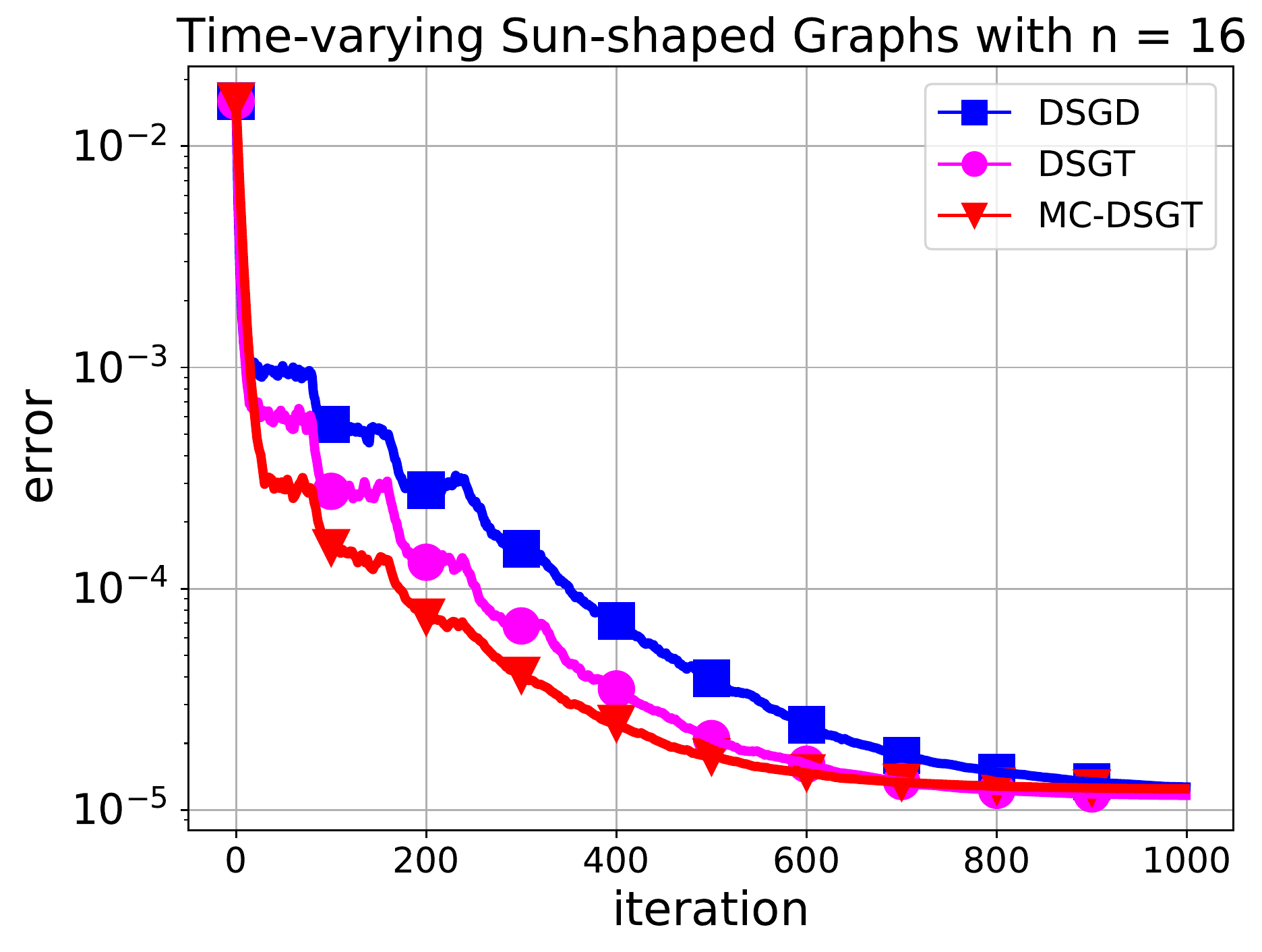}
    \quad 
    \includegraphics[width=0.3\textwidth]{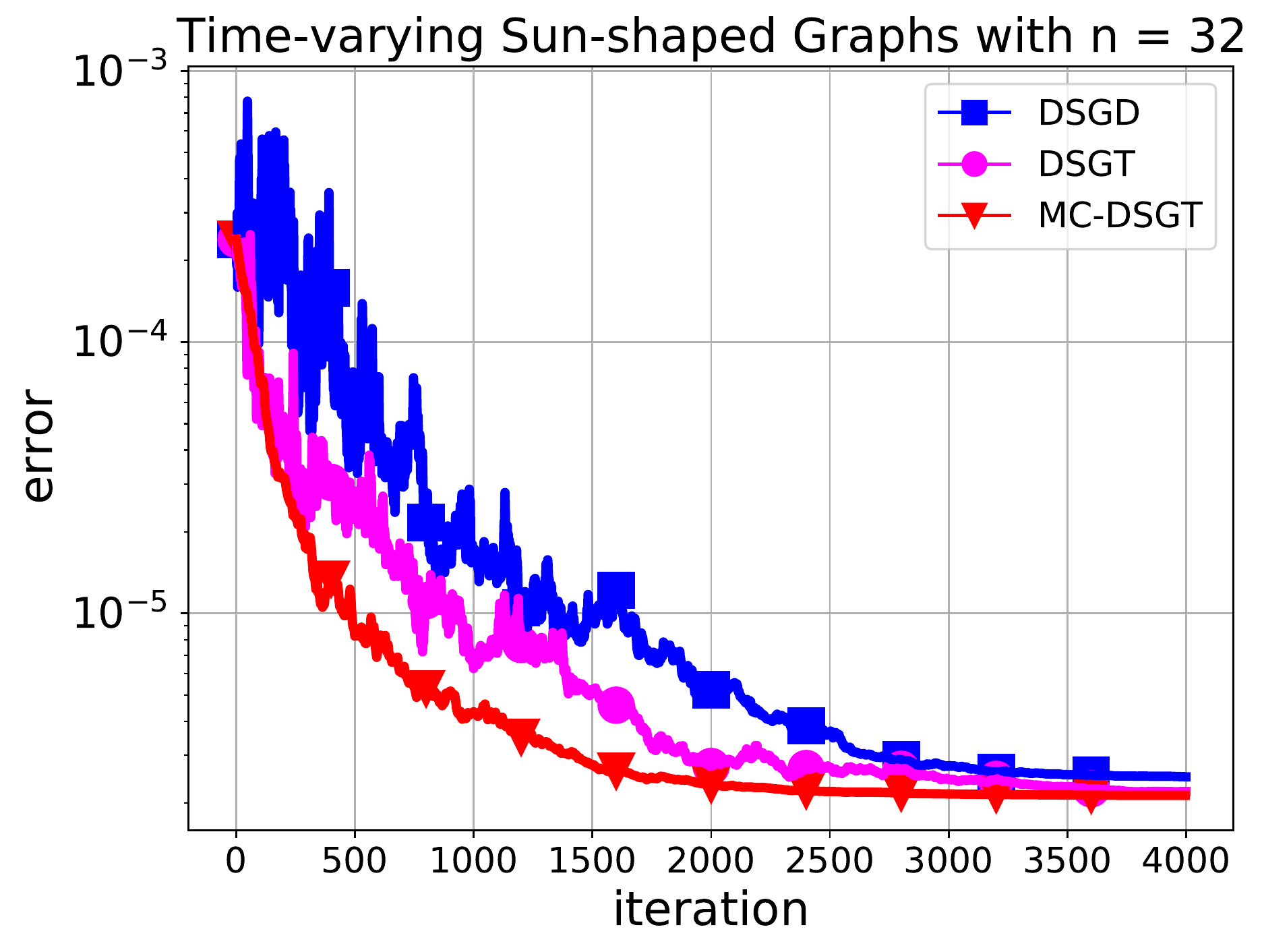}
    \vskip -3mm
    \caption{\small Performance of different stochastic algorithms to solve problem \eqref{eqn:nvwegdv}. The left plot is with MNIST, and the right plot is with COVTYPE.binary.}
    \label{fig:nc-simulation}
    \vskip -7mm
\end{figure}
We consider two real datasets: MNIST and COVTYPE.binary. We binarize MNIST dataset by considering datapoints with labels $2$ and $4$.  The
regularization weight $\rho$ is chosen as $0.2$ (MNIST) and $0.015$ (COVTYPE.binary).
We partition the two datasets non-uniformly such that a half of the nodes contain $80\%$ positive datapoints while the other half hold $80\%$ negative datapoints. We compare decentralized stochastic gradient descent (DSGD) \cite{koloskova2020unified}, decentralized stochastic gradient tracking (DSGT) \cite{xin2020improved} and Algorithm \ref{alg:MG-DSGT} (MC-DSGT) with random time-varying sun-shaped graphs with $(n,|\cC|)$ equal to $(16,1)$ for MNIST and $(32,4)$ for COVTYPE.binary. We set $R=2$ and $4$ in MC-DSGT for MNIST and COVTYPE.binary, repspectively

The performance of algorithms over MNIST and COVTYPE.binary is illustrated in the left and right plot in Figure \ref{fig:nc-simulation}, respectively. The error metric is taken as $\|\nabla f(\bar{x})\|^2$  with $\bar{x}=\frac{1}{n}\sum_{i=1}^nx_i^{(k)}$. In both experiments, we find the convergence rate as well as the robustness to time-varying network topology of MC-DSGT outperforms DSGD and DSGT, which coincides with our theory.

\section{Conclusion}
This paper provides the first optimal complexity for non-convex decentralized stochastic optimization over time-varying networks. We also generalize DSGT with multiple consensus under time-varying networks to match the optimal bound up to logarithm factors. Future works include establishing the optimal rate for (strongly) convex stochastic scenarios over time-varying networks.


\bibliography{references}

\appendix

\section{Sun-shaped Graph}
\begin{proof}[Proof of Theorem \ref{thm:sun-key}]
It is easy to see that when $|\cI_1|+|\cI_2|=n$, $\dist_{\{\cS_{n,\cC^t}\}_{t=0}^\infty}(\cI_1,\cI_2)=1$ for any graphs $\{\cS_{n,\cC^t}\}_{t=0}^\infty$. Thus in this case, we can simply let $\cC^t=[n]$ and $W^t=\beta I_n+(1-\beta)\one_n\one_n^\top$ for any $t\geq 0$. It is easy to see that $W^t\in\cW_{n, \beta}$.

Next we consider $|\cI_1|+|\cI_2|<n$. Let $k= \lceil n(1-\beta) \rceil\in[1,n]$. 

Case 1. If $k=n$, \ie, $0\leq \beta <\frac{1}{n}$, then we again let $\cC^t=[n]$  with associate weight matrix $W^t=\beta I_n+(1-\beta)\one_n\one_n^\top$ for all $t\geq 0$. It is easy to see that 
\begin{equation*}
    \dist_{\{\cS_{n,\cC^t}\}_{t=0}^\infty}(\cI_1,\cI_2)=1=\Theta(1)=\Theta\left(\frac{1-(|\cI_1|+|\cI_2|)/n}{1-\beta}+1\right)
\end{equation*}
where the last identity is because $0\leq 1-(|\cI_1|+|\cI_2|)/n\leq 1$ and $(1-\beta)^{-1}=\Theta(1)$.

Case 2. If $1\leq k\leq n-1$, then $\frac{1}{n}\leq \beta \leq 1-\frac{1}{n}$.
Let $\cJ^0,\dots,\cJ^{p-1}$ with $p=\lfloor (n-|\cI_1|-|\cI_2|)/k\rfloor$ be disjoint subsets of $[n]\backslash(\cI_1\cup \cI_2)$ such that each $\cJ^q$ ($0\leq q\leq p-1$) exactly contains $k$ nodes. Such $\{\cJ^{q}\}_{q=0}^{p-1}$ always exists due to  $p\times k\leq n-|\cI_1|-|\cI_2|$. Now let $\cC^t=\cJ^{t\text{ mod }p}$, \ie,
\begin{equation*}
    \{\cS_{n,\cC^t}\}_{t=0}^\infty=\{\cS_{n,\cJ^0},\dots,\cS_{n,\cJ^{p-1}},\cS_{n,\cJ^0}\dots,\cS_{n,\cJ^{p-1}},\cS_{n,\cJ^0},\dots\}.
\end{equation*}
It is easy to see that for any center set $\cC$ with $|\cC|=k$, the Laplacian $L(\cS_{n,\cC})$ of graph $\cS_{n,\cC}$ has eigenvalues: 
\begin{equation*}
    0,\,\underbrace{k,\,\dots,\,k}_{(n-k-1)\text{-folds}}, \,\underbrace{n,\,\dots,\,n}_{k\text{-folds}}.
\end{equation*}
We thus let the associated weight matrices to be $W^t=I_n-\frac{\delta}{n}L(\cS_{n,\cC^t})$ with $\delta ={n(1-\beta)}/{\lceil n(1-\beta) \rceil}\in(0,1]$ for any $t\geq 0$. Since $\delta <1$, $\{W^t\}_{t=0}^\infty$ are positive semi-definite. Therefore, we have 
\begin{equation*}
    \left\|W^t-\frac{1}{n}\one_n\one_n^\top\right\|=1-\frac{\delta k}{n}=1-\frac{n(1-\beta)}{n}=\beta.
\end{equation*}
The rest is to verify $\dist_{\{\cS_{n,\cC^t}\}_{t=0}^\infty}(\cI_1,\cI_2)= \Theta\left(\frac{1-(|\cI_1|+|\cI_2|)/n}{1-\beta}+1\right)$. By the construction of sun-shaped graphs, starting from any round $t$, the neighborhood of $\cI_1$ (or $\cI_2$) satisfies \begin{equation*}
    {\cN}_{\cS_{n,\cC^{t}}}({\cN}_{\cS_{n,\cC^{t+1}}}(\cdots{\cN}_{\cS_{n,\cC^{t+R-1}}}(\cI_1) \cdots))=\begin{cases}
    \left(\bigcup_{t^\prime=t}^{t+R-1}\cJ^{t\text{ mod }p}\right)\cup\cI_1&\text{if }R\leq p;\\
    [n]&\text{if }R>p+1.
    \end{cases}
\end{equation*}
Therefore, we conclude that 
\begin{align}\label{eqn:gejfwgcsd}
    \dist_{\{\cS_{n,\cC^t}\}_{t=0}^\infty}(\cI_1,\cI_2)=&p+1=\lfloor (n-|\cI_1|-|\cI_2|)/k\rfloor+1=\left\lfloor \frac{n-|\cI_1|-|\cI_2|}{\lceil n(1-\beta) \rceil}\right\rfloor+1.
\end{align}
On one hand, we easily see
\begin{equation}\label{eqn:vidnfsda}
    \left\lfloor \frac{n-|\cI_1|-|\cI_2|}{\lceil n(1-\beta) \rceil}\right\rfloor\leq \frac{n-|\cI_1|-|\cI_2|}{ n(1-\beta)}.
\end{equation}
On the other hand, since $n(1-\beta)\geq 1$, we have $\lceil n(1-\beta) \rceil\leq 2n(1-\beta)$ and further
\begin{equation}\label{eqn:vidnfsda2}
    \left\lfloor \frac{n-|\cI_1|-|\cI_2|}{\lceil n(1-\beta) \rceil}\right\rfloor+1\geq \left\lfloor \frac{n-|\cI_1|-|\cI_2|}{2 n(1-\beta)}\right\rfloor+1=\Omega\left( \frac{n-|\cI_1|-|\cI_2|}{2 n(1-\beta)}+1\right)
\end{equation}
where the last step is due to $\lfloor x\rfloor+1\geq (x+1)/2$ for any $x\geq 0$.
Combining \eqref{eqn:vidnfsda} and \eqref{eqn:vidnfsda2} with \eqref{eqn:gejfwgcsd}, we reach $ \dist_{\{\cS_{n,\cC^t}\}_{t=0}^\infty}(\cI_1,\cI_2)=\Theta\left(\frac{1-(|\cI_1|+|\cI_2|)/n}{1-\beta}+1\right)$.
\end{proof}

\section{Lower Bound}\label{app:lower-bounds}

\subsection{Proof of Theorem \ref{thm:lower-bound-nc}}
Without loss of generality, we assume algorithms to start from $x^{(0)}=0$. We denote the $j$-th coordinate of a vector $x\in\RR^d$ by $[x]_j$ for $j=1,\dots,d$, and let $\prog(x)$ be 
\begin{equation*}
    \prog(x):=\begin{cases}
    0 & \text{if $x=0$};\\
    \max_{1\leq j\leq d}\{j:[x]_j\neq 0\}& \text{otherwise}.
    \end{cases}
\end{equation*}
Similarly, for a set of points $\cX=\{x_1,x_2,\dots\}$, we define $\prog(\cX):=\max_{x\in\cX}\prog(x)$.
As described in \cite{carmon2020lower,carmon2021lower}, a zero chain function $f$ satisfies
\begin{equation*}
    \prog(\nabla f(x))\leq \prog(x)+1,\quad\forall\,x\in\RR^d,
\end{equation*}
which implies that, starting from $x=0$, a single gradient evaluation can only make at most one more coordinate for the model parameter $x$ be non-zero.

We prove the two terms of the lower bound in  Theorem \ref{thm:lower-bound-nc} separately by constructing two hard-to-optimize instances. We first state some key zero-chain functions that will be used to facilitate the analysis.
\begin{lemma}[Lemma 2 of \cite{Arjevani2019LowerBF}]\label{lem:basic-fun}
Let $[x]_j$  denote the $j$-th coordinate of a vector $x\in\RR^d$, and define function 
\begin{equation*}
    h(x):=-\psi(1) \phi([x]_{1})+\sum_{j=1}^{d-1}\Big(\psi(-[x]_j) \phi(-[x]_{j+1})-\psi([x]_j) \phi([x]_{j+1})\Big)
\end{equation*}
where for $\forall\, z \in \mathbb{R},$
$$
\psi(z)=\begin{cases}
0 & z \leq 1 / 2; \\
\exp \left(1-\frac{1}{(2 z-1)^{2}}\right) & z>1 / 2,
\end{cases} \quad \phi(z)=\sqrt{e} \int_{-\infty}^{z} e^{\frac{1}{2} t^{2}} \mathrm{d}t.
$$
Then $h$ satisfy the following properties:
\begin{enumerate}
    \item $h(x)-\inf_{x} h(x)\leq \delta_0 d$, $\forall\,x\in\RR^d$ with $\delta_0=12$;
    \item $h$ is $\ell_0$-smooth with $\ell_0=152$;
    \item $\|\nabla h(x)\|_\infty\leq g_\infty $, $\forall\,x\in\RR^d$ with $g_\infty = 23$;
    \item $\|\nabla h(x)\|_\infty\ge 1 $ for any $x\in\RR^d$ with $[x]_d=0$. 
\end{enumerate}
\end{lemma}

\begin{lemma}[Lemma 4 of \cite{Huang2022LowerBA}]\label{lem:basic-fun2}
Let functions 
\begin{equation*}
    h_1(x):=-2\psi(1) \phi([x]_{1})+2\sum_{j \text{ even, } 0< j<d}\Big(\psi(-[x]_j) \phi(-[x]_{j+1})-\psi([x]_j) \phi([x]_{j+1})\Big)
\end{equation*}
and 
\begin{equation*}
    h_2(x):=2\sum_{j \text{ odd, } 0<j<d}\Big(\psi(-[x]_j) \phi(-[x]_{j+1})-\psi([x]_j) \phi([x]_{j+1})\Big).
\end{equation*}
Then $h_1$ and $h_2$ satisfy  the following properties:
\begin{enumerate}
    \item $\frac{1}{2}(h_1+h_2)=h$, where $h$ is defined in Lemma \ref{lem:basic-fun}.
    \item For any $x\in\RR^d$, if $\prog(x)$ is odd, then $\prog(\nabla h_1(x))\leq \prog(x)$; if $\prog(x)$ is even, then $\prog(\nabla h_2(x))\leq \prog(x)$.
    \item $h_1$ and $h_2$ are also $\ell_0$-smooth with ${\ell_0}=152$. 
\end{enumerate}
\end{lemma}

Given Lemmas \ref{lem:basic-fun} and \ref{lem:basic-fun2}, we now construct two instances that lead to the two terms in lower bound \eqref{eqn:lower-bound-nc}, respectively.

\paragraph{Instance 1. } 
The proof of the first term $\Omega((\frac{\Delta L\sigma^2}{nT})^\frac{1}{2})$ essentially follows the first example in proving Theorem 1 of \cite{lu2021optimal}. We provide the key steps
for the sake of being self-contained.

(Step 1.) Let $f_i=L\lambda^2h(x/\lambda)/\ell_0$, $\forall\,i\in[n]$ be homogeneous and hence $f=L\lambda^2h(x/\lambda)/\ell_0$ where $h$ is defined in Lemma \ref{lem:basic-fun}  and $\lambda>0$ is to be specified. Since $\nabla^2 f_i=L\nabla^2 h/\ell_0$ and $h$ is $\ell_0$-smooth by Lemma \ref{lem:basic-fun}, we know $f_i$ is $L$-smooth for any $\lambda>0$. By Lemma \ref{lem:basic-fun}, we have
\begin{equation*}
    f(0)-\inf_x f(x)=\frac{L\lambda^2}{\ell_0^2}(h(0)-\inf_x h(x)) {\leq}\frac{L\lambda^2\delta_0d}{\ell_0}. 
\end{equation*}
Therefore, to ensure $f_i\in\cF_{L}^\Delta$, it suffices to let 
\begin{equation}\label{eqn:jgowemw}
    \frac{L\lambda^2\delta_0d}{\ell_0}\leq \Delta, \quad \text{i.e.,}\quad d\lambda^2\leq \frac{\ell_0 \Delta}{L\delta_0}.
\end{equation}

(Step 2.) We construct the stochastic gradient oracle $O_i$ on worker $i$, $\forall\,i\in[n]$ as the follows:
\begin{equation*}
    [O_i(x;Z)]_j=[\nabla f_i(x)]_j\left(1+\mathds{1}{\{j>\prog(x)\}}\left(\frac{Z}{p}-1\right)\right), \forall\,x\in\RR^d,\,j=1,\dots,d
\end{equation*}
with random variable $Z\sim \text{Bernoulii}(p)$ independent of $x$ and $f_i$, and $p\in(0,1)$ to be specified.
It is easy to see $O_i$ is an unbiased stochastic gradient oracle. Moreover, since $f_i$ is zero-chain, we have $\prog(O_i(x;Z))\leq \prog(\nabla f_i(x))\leq \prog(x)+1$ and hence
\begin{align*}
\EE[\|[O_i(x;Z)]-\nabla f_i(x)\|^2]&=|[\nabla f_i(x)]_{\prog(x)+1}|^2\EE\left[\left(\frac{Z}{p}-1\right)^2\right]=|[\nabla f_i(x)]_{\prog(x)+1}|^2\frac{1-p}{p}\\
&\leq \|\nabla f_i(x)\|_\infty^2\frac{1-p}{p}
\leq \frac{L^2\lambda^2(1-p)}{\ell_0^2p}\|\nabla h(x)\|_\infty^2\\
&\overset{\text{Lemma \ref{lem:basic-fun}}}{\leq } \frac{L^2\lambda^2(1-p)g_\infty^2}{\ell_0^2p}.
\end{align*}
Therefore, to ensure $O_i\in\cO_{\sigma^2}$, it suffices to let 
\begin{equation}\label{eqn:gjvownefoq}
     p=\min\{\frac{L^2\lambda^2g_\infty^2}{\ell_0^2\sigma^2},1\}.
\end{equation}

(Step 3.) 
Let $x^{(t)}_i$, $\forall\,t\geq 0$ and $i\in[n]$, be the $t$-th query point of worker $i$.
Since algorithms satisfy the zero-respecting property, as discussed in \cite{carmon2020lower,carmon2021lower,lu2021optimal}, within $T$ gradient queries on each worker, algorithms  can only return model $\hat{x}$ such that
\begin{equation*}
    \hat{x}\in\mathrm{span}\left(\left\{x^{(0)},\nabla f_i(x^{(0)}),\big\{\{x^{(t)}_i,\nabla f_i(x^{(t)}_i):0\leq t<T\}:1\le i\leq n\big\}\right\}\right),
\end{equation*}
which implies
\begin{equation}
    \prog(\hat{x})\leq \max_{0\leq t< T}\max_{1\leq i\leq n}\prog(x^{(t)}_i)+1.
\end{equation}
By Lemma 2 of \cite{lu2021optimal}, we have 
\begin{equation}\label{eqn:vjowemnfq}
    \PP( \prog(\hat{x})\ge d)\leq \PP\left(\max_{0\leq t< T}\max_{1\leq i\leq n}\prog(x^{(t)}_i)\geq d-1\right)\leq e^{(e-1)npT-d+1}.
\end{equation}
On the other hand, when $\prog(\hat{x})< d$, by Lemma \ref{lem:basic-fun}, it holds that 
\begin{align}\label{eqn:vjowemnfq2}
    \min_{\hat{x}\in\mathrm{span}\{\{x_i^{(t)}:1\leq i\leq n,\,0\leq t< T\}\}}\|\nabla f(\hat{x})\|\geq \min_{[\hat{x}]_{d}=0}\|\nabla f(\hat{x})\|
    =\frac{L\lambda}{\ell_0}\min_{[\hat{x}]_{d}=0}\|\nabla h(\hat{x})\|\geq \frac{L\lambda}{\ell_0}.
\end{align}                                                       
Therefore, by combining \eqref{eqn:vjowemnfq} and \eqref{eqn:vjowemnfq2}, we have
\begin{equation}\label{eqn:Lvjowenfq}
    \EE[\|\nabla f(\hat{x})\|^2]\geq \PP(\prog^{(T)}< d)\EE[\|\nabla f(\hat{x})\|^2\mid \prog^{(T)}< d]\geq (1-e^{(e-1)npT-d+1})\frac{L^2\lambda^2}{\ell_0^2}.
\end{equation}

Let 
\begin{equation*}
    \lambda=\frac{\ell_0}{L}\left(\frac{\Delta L\sigma^2}{3nT  \ell_0\delta_0 g_\infty^2}\right)^\frac{1}{4}\quad \text{and}\quad d=\left\lfloor\left(\frac{3L\Delta nT g_\infty^2}{\sigma^2\ell_0\delta_0}\right)^\frac{1}{2}\right\rfloor.
\end{equation*}
Then \eqref{eqn:jgowemw} naturally holds and $p=\min\{\frac{g_\infty^2}{\sigma^2}\left(\frac{\Delta L\sigma^2}{3nT \ell_0\delta_0 g_\infty^2}\right)^\frac{1}{2},1\}$ by \eqref{eqn:gjvownefoq}. Without loss of generality, we assume   $d\geq 2$, which is guaranteed when $T=\Omega(\frac{\sigma^2}{nL\Delta})$. Then, using the definition of $p$, we have that
\begin{align}
    &(e-1)npT-d+1\leq (e-1)nT\; \frac{g_\infty^2}{\sigma^2}\left(\frac{\Delta L\sigma^2}{3nT\ell_0\delta_0 g_\infty^2}\right)^\frac{1}{2}-d+1\nonumber\\
    =&\frac{e-1}{3}\left(\frac{3 L\Delta nT g_\infty^2}{\sigma^2 \ell_0\delta_0 }\right)^\frac{1}{2}-d+1< \frac{e-1}{3}(d+1)-d+1\leq  2-e<0\label{eqn:Lvjobqwm}\nonumber
\end{align}
which, combined with \eqref{eqn:Lvjowenfq}, leads to
\begin{equation*}
    \EE[\|\nabla f(\hat{x})\|^2]=\Omega\left(\frac{L^2\lambda^2}{\ell_0^2}\right)=\Omega\left(\left(\frac{\Delta L\sigma^2}{3nT  \ell_0\delta_0 g_\infty^2}\right)^\frac{1}{2}\right)=\Omega\left(\left(\frac{\Delta L\sigma^2}{nT}\right)^\frac{1}{2}\right).
\end{equation*}

\paragraph{Instance 2.} The proof for the second term  $\Omega(c{\Delta L}{T(1-\beta)})$ utilizes weight matrices defined on the sun-shaped graphs  described in Theorem \ref{thm:sun-key}.

(Step 1.) Let functions 
\begin{equation*}
    \ell_1(x):=-\frac{n}{\lceil n/4\rceil}\psi(1) \phi([x]_{1})+\frac{n}{\lceil n/4\rceil}\sum_{j \text{ even, } 0< j<d}\Big(\psi(-[x]_j) \phi(-[x]_{j+1})-\psi([x]_j) \phi([x]_{j+1})\Big)
\end{equation*}
and 
\begin{equation*}
    \ell_2(x):=\frac{n}{\lceil n/4\rceil}\sum_{j \text{ odd, } 0<j<d}\Big(\psi(-[x]_j) \phi(-[x]_{j+1})-\psi([x]_j) \phi([x]_{j+1})\Big).
\end{equation*}
By Lemma \ref{lem:basic-fun2}, 
$\ell_1$ and $\ell_2$ defined here are $2\ell_0$-smooth.
Furthermore, let
\begin{equation*}
    f_i=\begin{cases}
    L\lambda^2 \ell_1(x/\lambda)/(2\ell_0)&\text{if }i\in \mathcal{I}_1\triangleq \{j:1\leq j\leq \lceil \frac{n}{4}\rceil\},\\
    L\lambda^2 \ell_2(x/\lambda)/(2\ell_0)&\text{if }i\in\mathcal{I}_2\triangleq\{j:n-\lceil \frac{n}{4}\rceil+1\leq j\leq n\},\\
    0&\text{else.}
    \end{cases}
\end{equation*}
where $\lambda>0$ is to be specified.
To ensure $f_i\in\cF_{L}^\Delta$ for all $1\leq i\leq n$, it suffices to let
\begin{equation}\label{eqn:jgowemw-dfshadasd}
    \frac{L\lambda^2\Delta_0d}{2\ell_0}\leq \Delta, \quad \ie,\quad d\lambda^2\leq \frac{2\ell_0 \Delta}{L\Delta_0}.
\end{equation}
With the functions defined above, we have  $f(x)=\frac{1}{n}\sum_{i=1}^n f_i(x)=L\lambda^2 \ell(x/\lambda)/(2\ell_0)$ and 
\begin{align*}
    \prog(\nabla f_i(x))
    \begin{cases}
    =\prog(x) +1&\text{if } \{\prog(x) \text{ is even and } i\in \cI_1\}\cup\{\prog(x) \text{ is odd and }i\in \cI_2\}\\
    \leq \prog(x) &\text{otherwise}.
    \end{cases}
\end{align*}
Therefore, to make progress, \ie, to increase $\prog(x)$, for any gossip algorithm $A$, it must take the gossip communications to transmit information between $\cI_1$ to $\cI_2$ alternatively. Namely, it takes at least $\dist_{\{G^t\}_{t=0}^\infty}(\cI_1,\cI_2)$ rounds of decentralized communications for any possible gossip algorithm $A$ to increase $\prog(\hat{x})$ by $1$. Therefore, we have
\begin{equation}\label{eqn:jofqsfdq}
    \prog(\hat{x})\leq \max_{1\leq i\leq n,\,0\leq t< T}\prog(x^{(t)}_i)\leq \left\lfloor \frac{T}{\dist_{\{G^t\}_{t=0}^\infty}(\cI_1,\cI_2)}\right\rfloor+1,\quad \forall\,T\geq 0.
\end{equation}

(Step 2.) 
We consider a gradient oracle that return lossless full-batch gradients, i.e., $O_i(x)=\nabla f_i(x)$, $\forall\,x\in\RR^d, \,i\in[n]$. 
For the construction of graphs and  weight matrices, 
we consider the sequence of sun-shaped graphs $\{G^t:=\cS_{n,\cC^t}\}_{t=0}^\infty$ and their associated weight matrices $\{W^t\}_{t=0}^\infty\in\cW_{n,\beta}$ investigated in Theorem \ref{thm:sun-key}.
Since $1-(|\cI_1|+|\cI_2|)/n=\Omega(1)$, by Theorem \ref{thm:sun-key}, we have $\dist_{\{G^t\}_{t=0}^\infty}(\cI_1,\cI_2)=\Theta((1-\beta)^{-1})$. Suppose $\dist_{\{G^t\}_{t=0}^\infty}(\cI_1,\cI_2)\geq 1 /(C(1-\beta))$ with some absolute constant $C$, then by \eqref{eqn:jofqsfdq}, we have
\begin{equation}\label{eqn:jofqsfdq-0gsfa}
    \prog(\hat{x})\leq \left\lfloor {C(1-\beta) T}\right\rfloor+1,\quad \forall\,T\geq 0.
\end{equation}

(Step 3.) We finally show the error $\EE[\|\nabla f(x)\|^2]$ is lower bounded by $\Omega\left(\frac{\Delta L}{(1-\beta)T}\right)$, with any algorithm $A\in \mathcal{A}_{\{W^t\}_{t=0}^\infty}$. 
For any $T\geq 1/(C(1-\beta))=\Omega((1-\beta)^{-1})$, let
\begin{equation*}
    d= \left\lfloor C(1-\beta)T\right\rfloor+2 <3C(1-\beta)T
\end{equation*}
and
\begin{equation}\label{eqn:gjofefqfq}
     \lambda =\frac{L_0}{L}\sqrt{\frac{2\Delta L}{3C(1-\beta)TL_0\Delta_0}} .
\end{equation}
Then \eqref{eqn:jgowemw-dfshadasd} naturally holds.
Since $\prog(\hat{x})<d$ by \eqref{eqn:jofqsfdq-0gsfa}, following \eqref{eqn:vjowemnfq2} and using \eqref{eqn:gjofefqfq}, we have 
\begin{equation*}
    \EE[\|\nabla f(\hat{x})\|^2]\geq\min_{[\hat{x}]_{d}=0}\|\nabla f(\hat{x})\|^2\geq  \frac{L^2\lambda^2}{L_0^2}=\Omega\left(\frac{\Delta L}{(1-\beta)T}\right).
\end{equation*}

\section{Upper Bound}\label{app:uppe-bound}
\subsection{Preliminary}
\label{app-preliminary}
\noindent \textbf{Notation.} We first introduce necessary notations as follows.
\begin{itemize}
	\item $\vx^{(k)} = [(x_1^{(k)})^\top; (x_2^{(k)})^\top; \cdots; (x_n^{(k)})^\top]\in \mathbb{R}^{n\times d}$;
	\item $\tilde{\vg}^{(k)}\triangleq\nabla F(\vx^{(k)};\bxi^{(k,r)}) = [\nabla F_1(x_1^{(k)};\xi_1^{(k,r)})^\top; \cdots; \nabla F_n(x_n^{(k)};\xi_n^{(k,r)})^\top]\in \mathbb{R}^{n\times d}$;
	\item $\nabla f(\vx^{(k)}) = [\nabla f_1(x_1^{(k)})^\top; \nabla f_2(x_2^{(k)})^\top; \cdots; \nabla f_n(x_n^{(k)})^\top]\in \mathbb{R}^{n\times d}$ ;
	\item $\bar{\vx}^{(k)} = [(\bar{x}^{(k)})^\top; (\bar{x}^{(k)})^\top; \cdots; (\bar{x}^{(k)})^\top]\in \mathbb{R}^{n\times d}$ where $\bar{x}^{(k)} = \frac{1}{n}\sum_{i=1}^n x_i^{(k)}$;
	\item $W^t=[w_{i,j}^t]\in \mathbb{R}^{n\times n}$ is the weight matrix;
	\item $\mathds{1}_n = [1,1,\cdots, 1]^\top \in \RR^n$;
	\item Given two matrices $\vx, \vh \in \RR^{n\times d}$, we define inner product $\langle \vx, \vh \rangle = \mathrm{tr}(\vx^T \vh)$ and the Frobenius norm $\|\vx\|_F^2 = \langle \vx, \vx \rangle$;
	\item Given $W\in \RR^{n\times n}$, we let $\|W\|_2 = \sigma_{\max}(W)$ where   $\sigma_{\max}(\cdot)$ denote the maximum sigular value.
\end{itemize}

\noindent \textbf{Smoothness.} Since each $f_i(x)$ is assumed to be $L$-smooth, it holds that $f(x) = \frac{1}{n}\sum_{i=1}^n f_i(x)$ is also $L$-smooth. As a result, the following inequality holds for any $x, y \in \mathbb{R}^d$:
\begin{align}
f_i(x) \leq f_i(y) +  \langle \nabla f_i(y), x- y \rangle+\frac{L}{2}\|x - y\|^2. \label{sdu-2}
\end{align}

\noindent \textbf{Gradient noise.} For stochastic gradient oracles satisfying Assumption \ref{asp:gd-noise}, by independence, it holds for any $k\geq 0$ and $R\geq 1$ that 
\begin{align}
\EE[\|\tilde{g}_i^{(k)}-\nabla f_i(x_i^{(k)})\|^2]\leq \frac{\sigma^2}{R}\quad \text{and}\quad\EE\left[\left\|\overline{\tilde{g}}^{(k)}-\frac{1}{n}\sum_{i=1}^n\nabla f(x_i^{(k)})\right\|^2\right]\leq \frac{\sigma^2}{nR}
\end{align}
where $\overline{\tilde{g}}^{(k)}\triangleq\frac{1}{n}\sum_{i=1}^n\tilde{g}_i^{(k)}=\frac{1}{nR}\sum_{i=1}^n\sum_{r=0}^{R-1}O_i(x_i^{(k)};\zeta_i^{(k,r)})$.

\noindent \textbf{Network weighting matrix.} Since each weight matrix $W^t\in \cW_{n,\beta}$, it holds that 
\begin{align}\label{network-inequaliy}
\left\|W^t - \frac{1}{n}\mathds{1}_n\mathds{1}_n^\top\right\|_2 \leq  \beta.
\end{align}
Following \eqref{network-inequaliy},  it holds for a sequence of weight matrices $W^{t_1},\dots,W^{t_2-1}$ that
\begin{equation}\label{eqn:MG-inequaliy}
    \left\|\prod_{t=t_1}^{t_2-1}W^t - \frac{1}{n}\mathds{1}_n\mathds{1}_n^\top\right\|_2 \leq  \beta^{t_2-t_1}.
\end{equation}
Therefore, when $t_2-t_1$ grows,  $\prod_{t=t_1}^{t_2-1}W^t$ exponentially converges to $\frac{1}{n}\one_n\one_n^\top$.

\noindent \textbf{Submultiplicativity of the Frobenius norm.} For any matrix $W\in \RR^{n\times n}$ and $\vz\in \RR^{n\times d}$, it holds that 
\begin{align}\label{submulti}
\|W\vz\|_F \le \|W\|_2 \|\vz\|_F.
\end{align}
To verify it, by letting $z_j$ be the $j$-th row of $\vz$, we have $\|W\vz\|_F^2 = \sum_{j=1}^d \|Wz_j\|_2^2 \le \sum_{j=1}^d \|W\|_2^2 \|z_j\|_2^2=\|W\|_2^2\|\vz\|_F^2$. 

\subsection{Proof of Theorem \ref{thm:MG-DSGT-rate-nc}}
 
Our proof is adapted from the proof of \cite[Theorem 1]{xin2020improved}, which presents the convergence rate of stochastic decentralized gradient tracking with single consensus operation and a static weight matrix. We generalize the proof to suit multiple consensus and time-varying weight matrices.

We use the matrix-form notations of the algorithm mostly for convenience.
At the beginning of phase $k$, the
three quantities of interests are $\vx^{(k)}$, $\vh^{(k)}$ and $\tilde{\vg}^{(k)}\triangleq \nabla F(\vx^{(k)};\bxi^{(k,r)})$, and the update rule for any $k\geq 0$ is
\begin{align}
    \vx^{(k+1)}&=\vW^{(2k)}_R(\vx^{(k+1)}-\gamma\vh^{(k)}),\label{eqn:dsgt-iter-1}\\
    \vh^{(k+1)}&=\vW^{(2k+1)}_R(\vh^{(k)}+\tilde{\vg}^{(k+1)}-\tilde{\vg}^{(k)})\label{eqn:dsgt-iter-2}
\end{align}
where $\vW^{(k)}_R\triangleq \prod_{t=kR}^{(k+1)R-1}W^t$ for any $k\geq 0$ and $R\geq 1$. 
By \eqref{eqn:MG-inequaliy}, we have $\|\vW^{(k)}_R-\one\one^\top/n\|_2\leq \beta^R$ for any $k\geq 0$. By multiplying $\one_n\one_n^\top/n$ to the left-side of \eqref{eqn:dsgt-iter-1} and \eqref{eqn:dsgt-iter-2}, we have 
\begin{align}
    \bar{x}^{(k+1)}&=\bar{x}^{(k+1)}-\gamma\bar{h}^{(k)})\nonumber,\\
    \bar{h}^{(k+1)}&=\bar{h}^{(k)}+\overline{\tilde{g}}^{(k+1)}-\overline{\tilde{g}}^{(k)}.\label{eqn:dsgt-ave-iter-2}
\end{align}
Since $\bar{h}^{(0)}=\overline{\tilde{g}}^{(0)}$, by iterating \eqref{eqn:dsgt-ave-iter-2} over $0,\dots,k-1$, it holds that $\bar{h}^{(k)}=\overline{\tilde{g}}^{(k)}$ for any $k\geq 0$. We will use the following descent lemma, which is adapted from  \cite[Lemma 3]{xin2020improved}.
\begin{lemma}[\sc Descent Lemma]\label{lem:descent}
Under Assumption \ref{asp:nc}, \ref{asp:gd-noise}, \ref{asp:weight-matrix}, if $0< \gamma\leq \frac{1}{2L}$, then we have for any $k\geq 0$,
\begin{equation*}
    \EE[f(\bar{x}^{(k+1)})]\leq \EE[f(\bar{x}^{(k)})]-\frac{\gamma}{2}\EE[\|\nabla f(\bar{x}^{(k)})\|^2]-\frac{\gamma}{4}\EE[\|\bar{g}^{(k)}\|^2]+\frac{\gamma L^2}{2n}\EE[\|\Pi\vx^{(k)}\|_F^2]+\frac{\gamma^2L\sigma^2}{2nR}.
\end{equation*}
where $\Pi\triangleq I-\frac{1}{n}\one_n\one_n^\top$.
\end{lemma}
By iterating Lemma \ref{lem:descent} over $k=0,\dots,K$, we obtain
\begin{align}
    &\frac{1}{K+1}\sum_{k=0}^{K}\EE[\|\nabla f(\bar{x}^{(k)})\|^2]\nonumber\\
    \leq& \frac{2\Delta}{\gamma(K+1)}+\frac{\gamma L\sigma^2}{nR}-\frac{1}{2(K+1)}\sum_{k=0}^{K}\EE[\|\bar{g}^{(k)}\|^2]+\frac{L^2}{n(K+1)}\sum_{k=0}^{K}\EE[\|\Pi \vx^{(k)}\|_F^2]\label{eqn:vdisnef}
\end{align}
where $\Delta \geq  f(x^{(0)})-\min_x f(x)$.

We next turn to bound the consensus error $\EE[\|\Pi\vx^{(k)}\|_F^2]$, which relies on the following recursion bound of consensus errors.
\begin{lemma}[\sc Recursion of Consensus Error]\label{lem:consensus}
Under Assumption \ref{asp:nc}, \ref{asp:gd-noise}, \ref{asp:weight-matrix}, denoting $\rho\triangleq\beta^R$, it holds for $0<\gamma\leq \frac{1-\rho^2}{24(1+\rho^2)L}$ that 
\begin{align*}
    \EE[\|\Pi\vx^{(k+1)}\|_F^2]&\leq \frac{2\rho^2}{1+\rho^2}\EE[\|\Pi\vx^{(k)}\|_F^2]+\frac{2\gamma^2\rho^2}{1-\rho^2}\EE[\|\Pi\vh^{(k)}\|_F^2]\\
    \EE[\|\Pi\vh^{(k+1)}\|_F^2]&\leq \frac{36\rho^2L^2}{1-\rho^2}\EE[\|\Pi\vx^{(k)}\|_F^2]+\frac{2\rho^2}{1+\rho^2}\EE[\|\Pi\vh^{(k)}\|_F^2]+\frac{12n\gamma^2\rho^2L^2}{1-\rho^2}\EE[\|\bar{g}^{(k)}\|^2]+6n\frac{\sigma^2}{R}.
\end{align*}
\end{lemma}
\begin{proof}
Multiplying $\Pi$ to the left side of \eqref{eqn:dsgt-iter-1} and \eqref{eqn:dsgt-iter-2}, we have 
\begin{align}
    \Pi\vx^{(k+1)}&=\Pi\vW^{(2k)}_R(\vx^{(k+1)}-\gamma\vh^{(k)}),\label{eqn:dsgt-iter-11}\\
    \Pi\vh^{(k+1)}&=\Pi\vW^{(2k+1)}_R(\vh^{(k)}+\tilde{\vg}^{(k+1)}-\tilde{\vg}^{(k)}).\label{eqn:dsgt-iter-21}
\end{align}
Therefore, following \eqref{eqn:dsgt-iter-11}, by using $\|\Pi\vW^{(2k)}_R\va\|_F\leq \rho\|\Pi\va\|_F$ for any $\va\in\RR^{n\times n}$ and $-\langle \va, \vb\rangle\leq \frac{1-\rho^2}{1+\rho^2}\|\va\|_F^2+\frac{1+\rho^2}{1-\rho^2}\|\vb\|_F^2$ for any $\va,\vb\in\RR^{n\times n}$, we have 
\begin{align*}
    \|\Pi\vx^{(k+1)}\|_F^2=&\|\Pi\vW^{(2k)}_R\vx^{(k)}\|_F^2-2\gamma\langle \Pi\vW^{(2k)}_R\vx^{(k)},\Pi\vW^{(2k)}_R \vh^{(k)}\rangle_F+\gamma^2\|\Pi\vW^{(2k)}_R\vh^{(k)}\|_F^2\\
    \leq & \rho^2\|\Pi\vx^{(k)}\|_F^2+\frac{\rho^2(1-\rho^2)}{1+\rho^2}\|\Pi\vx^{(k)}\|_F^2+\frac{\gamma^2\rho^2(1+\rho^2)}{1-\rho^2}\|\Pi\vh^{(k)}\|_F^2+\gamma^2\rho^2 \|\Pi\vh^{(k)}\|_F^2\\
    =&\frac{2\rho^2}{1+\rho^2}\|\Pi\vx^{(k)}\|_F^2+\frac{2\gamma^2\rho^2}{1-\rho^2}\|\Pi\vh^{(k)}\|_F^2.
\end{align*}
Following \eqref{eqn:dsgt-iter-21},  we can bound $\|\Pi\vh^{(k+1)}\|_F^2$ as follows:
\begin{align}
    \EE[\|\Pi\vh^{(k+1)}\|_F^2]=&\EE[\|\Pi\vW^{(2k+1)}_R\vh^{(k)}\|_F^2]+2\EE[\langle \Pi\vW^{(2k+1)}_R \vh^{(k)},\Pi\vW^{(2k+1)}_R(\tilde{\vg}^{(k+1)}-\tilde{\vg}^{(k)})\rangle_F]\nonumber\\
    &\quad +\EE[\|\Pi\vW^{(2k+1)}_R(\tilde{\vg}^{(k+1)}-\tilde{\vg}^{(k)})\|_F^2]\nonumber\\
    \leq &\rho^2\EE[\|\Pi\vh^{(k)}\|_F^2]+2\EE[\langle \Pi\vW^{(2k+1)}_R \vh^{(k)},\Pi\vW^{(2k+1)}_R(\nabla f(\vx^{(k+1)})-\tilde{\vg}^{(k)})\rangle_F]\nonumber\\
    &\quad +\rho^2\EE[\|\Pi(\tilde{\vg}^{(k+1)}-\tilde{\vg}^{(k)})\|_F^2]\nonumber\\
    = &\rho^2\EE[\|\Pi\vh^{(k)}\|_F^2]+2\EE[\langle \Pi\vW^{(2k+1)}_R \vh^{(k)},\Pi\vW^{(2k+1)}_R(\nabla f(\vx^{(k)})-\tilde{\vg}^{(k)})\rangle_F]\nonumber\\
    &\quad +2\EE[\langle \Pi\vW^{(2k+1)}_R \vh^{(k)},\Pi\vW^{(2k+1)}_R(\nabla f(\vx^{(k+1)})-\nabla f(\vx^{(k)}))\rangle_F]\nonumber\\
    &\quad +\rho^2\EE[\|\tilde{\vg}^{(k+1)}-\tilde{\vg}^{(k)}\|_F^2]\label{eqn:vinreve}
\end{align}
where the inequality follows $\|\Pi\vW^{(2k)}_R\va\|_F\leq \rho\|\Pi\va\|_F\leq \rho\|\va\|_F$ and $\EE[\tilde{\vg}^{(k+1)}\mid \vh^{(k)},\tilde{\vg}^{(k)}]=\nabla f(\vx^{(k+1)})$.
We next bound the terms in \eqref{eqn:vinreve} one by one.
By using the similar derivation to \cite[Lemma 5]{xin2020improved}, we can easily reach
\begin{align}
    \EE[\|\tilde{\vg}^{(k+1)}-\tilde{\vg}^{(k)}\|_F^2]=&\EE[\|\nabla f(\vx^{(k+1)})-\tilde{\vg}^{(k)}\|_F^2]+\EE[\|\tilde{\vg}^{(k+1)}-\nabla f(\vx^{(k+1)})\|_F^2]\nonumber\\
    \leq &2\EE[\|\nabla f(\vx^{(k+1)})-\nabla f(\vx^{(k)})\|_F^2]+2\EE[\|\nabla f(\vx^{(k)})-\tilde{\vg}^{(k)}\|_F^2]+\frac{n\sigma^2}{R}\nonumber\\
    \leq &2L^2\EE[\|\vx^{(k+1)}-\vx^{(k)}\|_F^2]+\frac{3n\sigma^2}{R}\label{eqn:voenoew}
\end{align}
and 
\begin{align}
    \EE[\|\vx^{(k+1)}-\vx^{(k)}\|_F^2]\leq &3\EE[\|\Pi\vx^{(k+1)}\|_F^2]+3\EE[\|\Pi\vx^{(k)}\|_F^2]+3\EE[\|\bar{\vx}^{(k+1)}-\bar{\vx}^{(k)}\|^2_F]\nonumber\\
    \leq &3\EE[\|\Pi\vx^{(k+1)}\|_F^2]+3\EE[\|\Pi\vx^{(k)}\|_F^2]+3\gamma^2\EE[\|\overline{\tilde{\vg}}^{(k)}\|_F^2]\nonumber\\
    \leq &9\EE[\|\Pi\vx^{(k)}\|_F^2]+6\gamma^2\rho^2 \EE[\|\Pi\vh^{(k)}\|_F^2]+3n\gamma^2\EE[\|\bar{g}^{(k)}\|^2]+\frac{3\gamma^2\sigma^2}{R}\label{eqn:voenoew1}
\end{align}
where we use $\bar{\vx}^{(k+1)}-\bar{\vx}^{(k)}=-\gamma \overline{\tilde{\vg}}^{(k)}$ and $\EE[\|\overline{\tilde{\vg}}^{(k)}\|_F^2]\leq \EE[\|\bar{\vg}^{(k)}\|_F^2]+{\sigma^2}/{R}=n\EE[\|\bar{g}^{(k)}\|^2]+{\sigma^2}/{R}$.   Combining \eqref{eqn:voenoew} and \eqref{eqn:voenoew1} together, we reach 
\begin{align}
    &\EE[\|\tilde{\vg}^{(k+1)}-\tilde{\vg}^{(k)}\|_F^2]\nonumber\\
    \leq& 18L^2\EE[\|\Pi\vx^{(k)}\|_F^2]+12\gamma^2\rho^2L^2 \EE[\|\Pi\vh^{(k)}\|_F^2]+6n\gamma^2L^2\EE[\|\bar{g}^{(k)}\|^2]+(3n+6n\gamma^2L^2)\frac{\sigma^2}{R}.\label{eqn:vnsvnso}
\end{align}
We next turn to bound $\EE[\langle \Pi\vW^{(2k+1)}_R \vh^{(k)},\Pi\vW^{(2k+1)}_R(\nabla f(\vx^{(k)})-\tilde{\vg}^{(k)})\rangle_F]$ in \eqref{eqn:vinreve}. For any $k\geq 1$, since $\vh^{(k)}=\vW_R^{(2k-1)R}(\vh^{(k-1)}+\tilde{\vg}^{(k)}-\tilde{\vg}^{(k-1)})$, $\EE[\nabla f(\vx^{(k)})-\tilde{\vg}^{(k)}\mid \vh^{(k-1)},\tilde{\vg}^{(k-1)}]=0$, we reach
\begin{align*}
    &\EE[\langle \Pi\vW^{(2k+1)}_R \vh^{(k)},\Pi\vW^{(2k+1)}_R(\nabla f(\vx^{(k)})-\tilde{\vg}^{(k)})\rangle_F]\\
    =&\EE[\langle \Pi\vW^{(2k+1)}_R\vW^{(2k-1)}_R \tilde{\vg}^{(k)},\Pi\vW^{(2k+1)}_R(\nabla f(\vx^{(k)})-\tilde{\vg}^{(k)})\rangle_F]\\
    =&\EE[\langle \Pi\vW^{(2k+1)}_R\vW^{(2k-1)}_R (\tilde{\vg}^{(k)}-\nabla f(\vx^{(k)})),\Pi\vW^{(2k+1)}_R(\nabla f(\vx^{(k)})-\tilde{\vg}^{(k)})\rangle_F].
\end{align*}
Since 
\begin{equation*}
    \left\|\left(\Pi\vW^{(2k+1)}_R\vW^{(2k-1)}_R\right)^\top\Pi \vW^{(2k+1)}_R\right\|_2=\left\| \left(\vW^{(2k+1)}_R\vW^{(2k-1)}_R\right)^\top\vW^{(2k+1)}_R-\frac{1}{n}\one_n\one_n^\top\right\|_2\leq \rho^3,
\end{equation*}
we further have
\begin{equation}
    \EE[\langle \Pi\vW^{(2k+1)}_R \vh^{(k)},\Pi\vW^{(2k+1)}_R(\nabla f(\vx^{(k)})-\tilde{\vg}^{(k)})\rangle_F]\leq \EE[\|\nabla f(\vx^{(k)})-\tilde{\vg}^{(k)}\|_F^2]\leq \frac{n\rho^2\sigma^2}{R}.\label{eqn:vnsvnsoo}
\end{equation}
It is easy to see that \eqref{eqn:vnsvnsoo} also holds for $k=0$.
We finally bound the last term \\$\EE[\langle \Pi\vW^{(2k+1)}_R \vh^{(k)},\Pi\vW^{(2k+1)}_R(\nabla f(\vx^{(k+1)})-\nabla f(\vx^{(k)}))\rangle_F]$ in \eqref{eqn:vinreve}. Since $\|\Pi\vW_R^{(2k+1)}\va\|_F\leq \rho\|\va\|_F$ for any $\va\in\RR^{n\times d}$, we have 
\begin{align}
    &\EE[\langle \Pi\vW^{(2k+1)}_R \vh^{(k)},\Pi\vW^{(2k+1)}_R(\nabla f(\vx^{(k+1)})-\nabla f(\vx^{(k)}))\rangle_F]\leq  \rho^2L\EE[\|\Pi\vh^{(k)}\|_F\|\vx^{(k+1)}-\vx^{(k)}\|_F]\nonumber\\
    \leq &\rho^2L\EE\left[\|\Pi\vh^{(k)}\|_F\left(\|\Pi\vx^{(k+1)}\|_F+\|\Pi\vx^{(k)}\|_F+\|\bar{\vx}^{(k+1)}-\bar{\vx}^{(k)}\|_F\right)\right]\nonumber\\
    \leq &\rho^2L\EE\left[\|\Pi\vh^{(k)}\|_F\left(2  \|\Pi\vx^{(k)}\|_F+\gamma\rho\|\Pi\vh^{k}\|_F+\gamma\|\overline{\tilde{\vg}}^{(k)}\|_F\right)\right]\label{eqn:nviewwvd}
\end{align}
where we us $\|\Pi\vx^{(k+1)}\|_F\leq \rho \|\Pi\vx^{(k)}\|_F+\gamma\rho\|\Pi\vh^{(k)}\|_F$ and $\bar{\vx}^{(k+1)}-\bar{\vx}^{(k)}=-\gamma \overline{\tilde{\vg}}^{(k)}$ in the last inequality.
By Young's inequality, we have for any $\eta_1,\eta_2> 0$  that
\begin{align}
    &\EE[\rho\|\Pi\vh^{(k)}\|_F\gamma\rho L\|\overline{\tilde{\vg}}^{(k)}\|_F]\nonumber\\
    \leq& 0.5\eta_1\rho^2\EE[\|\Pi\vh^{(k)}\|_F^2]+0.5\eta_1^{-1}\gamma^2\rho^2L^2\EE[\|\overline{\tilde{\vg}}^{(k)}\|_F^2]\nonumber\\
    \leq &0.5\eta_1\rho^2\EE[\|\Pi\vh^{(k)}\|_F^2]+0.5\eta_1^{-1}\gamma^2\rho^2L^2n\EE[\|{\bar{g}}^{(k)}\|^2]+0.5\eta_1^{-1}\gamma^2\rho^2L^2   \frac{\sigma^2}{R}\label{eqn:vdinviw1}
\end{align}
and 
\begin{align}
    2\EE[\rho\|\Pi\vh^{(k)}\|_F\rho L\|\Pi{\vx}^{(k)}\|_F]\leq& \eta_2\rho^2\EE[\|\Pi\vh^{(k)}\|]+\eta_2^{-1}\rho^2L^2\EE[\|\Pi\vx^{(k)}\|_F].\label{eqn:vdinviw2}
\end{align}
Plugging \eqref{eqn:vdinviw1} and \eqref{eqn:vdinviw2} into \eqref{eqn:nviewwvd}, we have
\begin{align}
    &\EE[\langle \Pi\vW^{(2k+1)}_R \vh^{(k)},\Pi\vW^{(2k+1)}_R(\nabla f(\vx^{(k+1)})-\nabla f(\vx^{(k)}))\rangle_F]\nonumber\\
    \leq&  \rho^2(\gamma\rho L+0.5\eta_1+\eta_2)\EE[\|\Pi\vh^{(k)}\|_F^2]+\eta_2^{-1}\rho^2L^2\EE[\|\Pi\vx^{(k)}\|_F^2]\nonumber\\
    &\quad +0.5\eta_1^{-1}\gamma^2\rho^2L^2n\EE[\|\bar{g}^{(k)}\|^2]+0.5\eta_1^{-1}\gamma^2\rho^2L^2   \frac{\sigma^2}{R}.\label{eqn:vnsvnsooo}
\end{align}
Plugging \eqref{eqn:vnsvnso}, \eqref{eqn:vnsvnsoo}, and \eqref{eqn:vnsvnsooo} into \eqref{eqn:vinreve}, we reach 
\begin{align}
    \EE[\|\Pi\vh^{(k+1)}\|_F^2]\leq &\rho^2(1+12\gamma^2\rho^2L^2+2\gamma\rho L +\eta_1+2\eta_2)\EE[\|\Pi\vh^{(k)}\|_F^2]\nonumber\\
    &\quad +\rho^2L^2(18+2\eta_2^{-1})\EE[\|\Pi\vx^{(k)}\|_F^2\nonumber\\
    &\quad +n\gamma^2\rho^2L^2(6+\eta_1^{-1})\EE[\|\bar{g}^{(k)}\|^2]+(5\rho^2n+2n\gamma^2\rho^2L^2+\eta_1^{-1}\gamma^2\rho^2L^2)\frac{\sigma^2}{R}.\label{eqn:vnidsnvw}
\end{align}
Letting $\eta_1=\frac{2(1-\rho^2)}{9(1+\rho^2)}$ and $\eta_2=\frac{1-\rho^2}{9(1+\rho^2)}$, then it holds for any $0\leq \gamma\leq \frac{1-\rho^2}{24(1+\rho^2) L}$ that 
\begin{align*}
    \rho^2(1+12\gamma^2\rho^2L^2+2\gamma\rho L +\eta_1+2\eta_2)&\leq \frac{2\rho^2}{1+\rho^2}\\
    \rho^2L^2(18+2\eta_2^{-1})&\leq \frac{36\rho^2L^2}{1-\rho^2   }\\
    n\gamma^2\rho^2L^2(6+\eta_1^{-1})&\leq \frac{12n\gamma^2\rho^2L^2}{1-\rho^2}\\
    5\rho^2n+2n\gamma^2\rho^2L^2+\eta_1^{-1}\gamma^2\rho^2L^2&\leq 6n,
\end{align*}
which, combined with \eqref{eqn:vnidsnvw}, leads to the conclusion.
\end{proof}

Letting $a^{(k)}\triangleq[\frac{1}{n}\EE[\|\Pi\vx^{(k)}\|_F^2], \frac{1}{nL^2}\EE[\|\Pi\vh^{(k)}\|_F^2]]^\top\in\RR^{2}$, $b^{(k)}\triangleq[0,\frac{12\gamma^2\rho^2}{1-\rho^2}\EE[\|\bar{g}^{(k)}\|_F^2]+\frac{6\sigma^2}{RL^2}]^\top\in\RR^2$ for any $k\geq 0$, and 
\begin{equation*}
    M\triangleq\begin{bmatrix}
    \frac{2\rho^2}{1+\rho^2} &\frac{2\gamma^2\rho^2L^2}{1-\rho^2}\\
    \frac{36\rho^2}{1-\rho^2}&\frac{2\rho^2}{1+\rho^2}
    \end{bmatrix},
\end{equation*}
by Lemma \ref{lem:consensus}, it holds that
\begin{equation*}
    a^{(k+1)}\preceq Ma^{(k)}+b^{(k)}
\end{equation*}
where $\preceq$ indicates entry-wise inequality.
Since $\gamma\leq \frac{(1-\rho^2)^2}{9\rho^2(1+\rho^2)L}$, one can check that there exists $v_1,v_2\geq 0$ such that $M[v_1,v_2]^\top \prec [v_1,v_2]^\top$. Therefore,
by \cite[Lemma 9]{xin2020improved}, we have for any $k\geq 0$ that
\begin{equation*}
    \sum_{\ell=0}^kM^k\preceq (I_{2\times 2}-M)^{-1}\preceq \begin{bmatrix}
    \frac{9(1+\rho^2)}{(1-\rho^2)} &\frac{18\gamma^2\rho^2(1+\rho^2)^2L^2}{(1-\rho^2)^3}\\
    \frac{324\rho^2(1+\rho^2)^2}{(1-\rho^2)^3}&
     \frac{9(1+\rho^2)}{(1-\rho^2)}
    \end{bmatrix}.
\end{equation*}
Therefore, we reach
\begin{align*}
    \sum_{k=0}^{K}a^{(k)}\preceq& \sum_{k=0}^{K}\left(M^ka^{(0)}+\sum_{\ell=0}^{k-1}M^{\ell}b^{(k-1-\ell)}\right)\\
    \preceq& \sum_{k=0}^\infty M^k\left(a^{(0)}+\sum_{k=0}^{K-1}b^{(k)}\right)\\
    \preceq &(I_{2\times 2}-M)^{-1}\left(a^{(0)}+\sum_{k=0}^{K-1}b^{(k)}\right).
\end{align*}
Since $a^{(0)}=0$ by our initialization, considering the first entry of the above, we have
\begin{equation}\label{eqn:bowgdsf}
        \frac{L^2}{n(K+1)}\sum_{k=0}^{K}\EE[\|\Pi\vx^{(k)}\|_F^2]\leq \frac{216\gamma^4\rho^4(1+\rho^2)^2L^4}{(1-\rho^2)^4(K+1)}\sum_{k=0}^K\EE[\|\bar{g}^{(k)}\|_F^2]+\frac{108\gamma^2\rho^2(1+\rho^2)^2L^2\sigma^2}{(1-\rho^2)^3R}.
\end{equation}
When $\gamma \leq \frac{1-\rho^2}{5\rho \sqrt{1+\rho^2}L}$, 
\begin{equation*}
    \frac{216\gamma^4\rho^4(1+\rho^2)^2L^4}{(1-\rho^2)^4(K+1)}\leq \frac{1}{2(K+1)}.
\end{equation*}
Hence, plugging \eqref{eqn:bowgdsf} into \eqref{eqn:vdisnef} yields 
\begin{align*}
    \frac{1}{K+1}\sum_{k=0}^{K}\EE[\|\nabla f(\bar{x}^{(k)})\|^2]\leq \frac{2\Delta}{\gamma(K+1)}+\frac{\gamma L\sigma^2}{nR}+\frac{108\gamma^2\rho^2(1+\rho^2)^2L^2\sigma^2}{(1-\rho^2)^3R}.
\end{align*}
Plugging
\begin{align}
    \gamma=&\min\left\{\frac{1}{2L}, \frac{1-\rho^2}{24(1+\rho^2)L},\frac{(1-\rho^2)^2}{9\rho^2(1+\rho^2)L},\frac{1-\rho^2}{5\rho \sqrt{1+\rho^2}L},\left(\frac{(1-\rho^2)^3R\Delta}{108 \rho^2(1+\rho^2)^2L^2\sigma^2(K+1)}\right)^\frac{1}{3}\right\}\label{eqn:lr}\\
    =&\Theta\left(\min\left\{\frac{1-\rho}{L},\frac{(1-\rho)^2}{\rho^2 L},\left(\frac{(1-\rho)^3R^2\Delta}{ \rho^2L^2\sigma^2T}\right)^\frac{1}{3}\right\}\right)\nonumber
\end{align}
 and $T=KR$ into the above, we reach
\begin{equation*}
     \frac{1}{K+1}\sum_{k=0}^{K}\EE[\|\nabla f(\bar{x}^{(k)})\|^2]=O\left(\left(\frac{\Delta L\sigma^2}{nT}\right)^\frac{1}{2}+\frac{R\Delta L}{T}+\left(\frac{\rho^2\Delta^2 L^2R\sigma^2}{(1-\rho)^3T^2}\right)^\frac{1}{3}+\frac{\rho^2R\Delta L}{T(1-\rho)^2}\right).
\end{equation*}
Furthermore, if one set 
\begin{equation}\label{eqn:R}
    R=\frac{1}{1-\beta}\max\left\{\ln(2),\ln\left(\frac{n^\frac{3}{4}L^\frac{1}{4}\Delta^\frac{1}{4}}{T^\frac{1}{4}(1-\beta)^\frac{1}{2}\sigma^\frac{1}{2}}\right)\right\}=\tilde{O}\left(\frac{1}{1-\beta}\right),
\end{equation}
so that 
\begin{equation*}
    \rho =\beta^R\leq e^{-(1-\beta)R}\leq \min\left\{\frac{1}{2},\frac{T^\frac{1}{4}(1-\beta)^\frac{1}{2}\sigma^\frac{1}{2}}{n^\frac{3}{4}L^\frac{1}{4}\Delta^\frac{1}{4}}\right\},
\end{equation*}
then we obtain 
\begin{equation*}
    \frac{1}{K+1}\sum_{k=0}^{K}\EE[\|\nabla f(\bar{x}^{(k)})\|^2]=\tilde{O}\left(\left(\frac{\Delta L\sigma^2}{nT}\right)^\frac{1}{2}+\frac{\Delta L}{T(1-\beta)}\right).
\end{equation*}

\end{document}